\documentclass{article}

\usepackage{graphicx}
\usepackage{arxiv}
\usepackage[numbers]{natbib}

\usepackage[utf8]{inputenc} 
\usepackage[T1]{fontenc}    
\usepackage{hyperref}       
\usepackage{url}            
\usepackage{color}
\usepackage{wrapfig}

\usepackage{amsmath}
\usepackage{amsfonts}
\usepackage{amssymb}
\usepackage{amsthm}
\usepackage[mathscr]{eucal}
\theoremstyle{definition}
\newtheorem{definition}{Definition}

\theoremstyle{theorem}
\newtheorem{theorem}{Theorem}
\newtheorem{proposition}{Proposition}
\newtheorem{lemma}{Lemma}
\newtheorem{remark}{Remark}

\newcommand{\thisHypName}{}
\newtheorem*{genericHyp}{\thisHypName}
\newenvironment{namedHyp}[1]
  {\renewcommand{\thisHypName}{#1}%
   \begin{genericHyp}}
  {\end{genericHyp}}

\newcommand{\AlgA}{Algorithm 1}
\newcommand{\AlgB}{Algorithm 2}
\newcommand{\AlgC}{Algorithm 3}
\newcommand{\AlgD}{Algorithm 4}

\newcommand{\nc}{\text{(NC)}}
\newcommand{\rcs}{\text{(RCS)}}

\newcommand{\semiMetric}{\text{dist}}
\newcommand{\semiDist}[2]{\semiMetric \left( #1 , #2 \right)}
\newcommand{\KLDiv}[2]{\mathcal{D}_{\text{kl}} \left(#1 | #2 \right)}

\newcommand{\dist}{\mathcal{F}}
\newcommand{\X}{\mathcal{X}}

\newcommand{\Hyp}{\mathcal{H}}
\newcommand{\Gyp}{\mathcal{G}}
\renewcommand{\H}{\Hyp}

\newcommand{\V}{d_\Hyp}
\newcommand{\Pdim}{d_{p}}
\newcommand{\VV}{d}
\newcommand{\E}{\mathcal{E}}

\newcommand{\expec}{\mathbb{E} }

\newcommand{\abs}[1]{\left|#1\right|}
\newcommand{\paren}[1]{\left(#1\right)}
\newcommand{\braces}[1]{\left\{#1\right\}}

\newcommand{\prob}{\mathbb{P}}

\usepackage{dsfont}
\DeclareSymbolFont{bbold}{U}{bbold}{m}{n}
\DeclareSymbolFontAlphabet{\mathbbold}{bbold}
\newcommand{\ind}{\mathbbold{1}}

\DeclareMathOperator*{\argmin}{argmin}

\DeclareMathOperator*{\Expectation}{\expec}
\DeclareMathOperator*{\Prob}{\prob}

\newcommand{\cost}{\mathfrak{c}}
\newcommand{\nats}{\mathbb{N}}

\newsavebox{\savepar}
\newenvironment{bigboxit}{\begin{center}\begin{lrbox}{\savepar}
\begin{minipage}[h]{4.6in}
\normalfont
\begin{flushleft}}
{\end{flushleft}\end{minipage}\end{lrbox}\fbox{\usebox{\savepar}}
\end{center}}

\newcommand{\Pclass}{{\mathscr{P}}}

\title{On the Value of Target Data in Transfer Learning}

\author{
 Steve Hanneke \\
   Toyota Technological Institute at Chicago\\
   \texttt{steve.hanneke@gmail.com}
   \And
 Samory Kpotufe \\
   Columbia University, Statistics\\
   \texttt{skk2175@columbia.edu}
}

\begin{document}
\maketitle

\begin{abstract}
We aim to understand the value of additional labeled or unlabeled target data in transfer learning, for any given amount of source data; this is motivated by practical questions around minimizing sampling costs, whereby, target data is usually harder or costlier to acquire than source data, but can yield better accuracy. 


To this aim, we establish the first minimax-rates in terms of both source and target sample sizes, and show that performance limits are captured by new notions of discrepancy between source and target, 
which we refer to as \emph{transfer exponents}. 

 Interestingly, we find that attaining minimax performance is akin to ignoring one of the source or target samples, provided distributional parameters were known a priori. Moreover, we show that practical decisions -- w.r.t. minimizing sampling costs -- can be made in a minimax-optimal way \emph{without} knowledge or estimation of distributional parameters nor of the discrepancy between source and target. 
\end{abstract}

\section{Introduction}
The practice of transfer-learning often involves acquiring some amount of target data, and involves various practical decisions as to how to best combine source and target data; however much of the theoretical literature on transfer only addresses the setting where no target labeled data is available. 

We aim to understand the value of target labels, that is, given $n_P$ labeled data from some source distribution $P$, and $n_Q$ labeled target labels from a target $Q$, what is the best $Q$ error achievable by any classifier in terms of \emph{both} $n_Q$ and $n_P$, and which classifiers achieve such optimal transfer. { In this first analysis, we mostly restrict ourselves to a setting, similar to the traditional \emph{covariate-shift} assumption, where the best classifier -- from a fixed VC class $\Hyp$ -- is the same under $P$ and $Q$.}

We establish the first minimax-rates, for bounded-VC classes, in terms of both source and target sample sizes $n_P$ and $n_Q$, and show that performance limits are captured by new notions of discrepancy between source and target, 
which we refer to as \emph{transfer exponents}. 

The first notion of transfer-exponent, called $\rho$, is defined in terms of discrepancies in excess risk, and is most refined. Already here, our analysis reveals a surprising fact: the best possible rate (matching upper and lower-bounds) in terms of $\rho$ and both sample sizes $n_P, n_Q$ is - up to constants - achievable by an oracle which simply ignores the 
least informative of the source or target datasets.
In other words, if $\hat h_P$ and $\hat h_Q$ denote the ERM on  data from $P$, resp. from $Q$, one of the two achieves the optimal $Q$ rate over any classifier having access to both $P$ and $Q$ datasets. However, which of $\hat h_P$ or $\hat h_Q$ is optimal is not easily decided without prior knowledge: for instance, cross-validating on a holdout target-sample would naively result in a rate of $n_Q^{-1/2}$ which can be far from optimal given large $n_P$. Interestingly, we show that the optimal $(n_P, n_Q)$-rate is achieved by a generic approach, akin to so-called \emph{hypothesis-transfer} \cite{kuzborskij2013stability, du2017hypothesis}, which optimizes $Q$-error under the constraint of low $P$-error, and does so without knowledge of distributional parameters such as $\rho$.

We then consider a related notion of \emph{marginal} transfer-exponent, called $\gamma$, defined w.r.t. marginals $P_X, Q_X$. 
This is motivated by the fact that practical decisions in transfer often involve the use of cheaper unlabeled data (i.e., data drawn from $P_X, Q_X$). We will show that, when practical decisions are driven by observed changes in marginals $P_X, Q_X$, the marginal notion $\gamma$ is then most suited to capture performance as it does not require knowledge (or observations) of label distribution $Q_{Y|X}$.


In particular, the marginal exponent $\gamma$ helps capture performance limits in the following scenarios of current practical interest: 

$\bullet$ {\bf Minimizing sampling cost.} Given different costs of labeled source and target data, and a desired target excess error at most $\epsilon$, how to use unlabeled data to decide on an optimal sampling scheme that minimizes labeling costs while achieving target error at most $\epsilon$. (Section \ref{sec:adaptive})

$\bullet$ {\bf Choice of transfer.} Given two sources $P_1$ and $P_2$, each at some unknown distance from $Q$, given unlabeled data and some or no labeled data from $Q$, how to decide which of $P_1, P_2$ transfers best to the target $Q$. 
(Appendix \ref{sec:P1vsP2})

$\bullet$ {\bf Reweighting.} Given some amount of unlabeled data from $Q$, and some or no labeled $Q$ data, how to optimally re-weight (out of a fixed set of schemes) the source $P$ data towards best target performance. While differently motivated, this problem is related to the last one. (Appendix  \ref{sec:reweighting})

Although optimal decisions in the above scenarios depend tightly on unknown distributional parameters such as different label noise in source and target data, and on unknown \emph{distance} from source to target (as captured by  $\gamma$), we show that such practical decisions can be made, near optimally, with no knowledge of distributional parameters, and perhaps surprisingly, without ever estimating $\gamma$. Furthermore, the unlabeled sampling complexity can be shown to remain low.  
Finally, the procedures described in this work remain of a theoretical nature, but 
yield new insights into how various practical decisions in transfer can be made near-optimally in a data-driven fashion.

\paragraph{Related Work.}
Much of the theoretical literature on transfer can be subdivided into a few main lines of work. As mentioned above, the main distinction with the present work is in that they mostly focus on situations with no labeled target data, and consider distinct notions of discrepancy between $P$ and $Q$. We contrast these various notions with the transfer-exponents $\rho$ and $\gamma$ in Section \ref{sec:examples}. 

A first direction considers refinements of total-variation that quantify changes in error over classifiers in a fixed class $\Hyp$. 
The most common such measures are the so-called $d_{\mathcal{A}}$-divergence \citep{ben2010theory, david2010impossibility, germain2013pac} and the $\mathcal{Y}$-discrepancy \citep{mansour2009domain, mohri2012new, cortesadaptation}. In this line of work, the rates of transfer, largely expressed in terms of $n_P$ alone, take the form $o_p(1) + C\cdot \text{divergence}(P, Q)$.  In other words, 
transfer down to $0$ error seems impossible whenever these divergences are non-negligible; we will carefully argue that such intuition can be overly pessimistic.

Another prominent line of work, which has led to many practical procedures, considers so-called density ratios $f_Q/f_P$ (importance weights) as a way to capture the similarity between $P$ and $Q$ \citep{quionero2009dataset, sugiyama2012density}. %
A related line of work considers information-theoretic measures such as KL-divergence or Renyi divergence 
\citep{sugiyama2008direct, mansour2009multiple} but has received relatively less attention. Similar to these notions, the transfer-exponents $\rho$ and $\gamma$ are \emph{asymmetric} measures of distance, attesting to the fact that it could be easier to transfer from some $P$ to $Q$ than the other way around. 
However, a significant downside to these notions is that they do not account for the specific 
structure of a hypothesis class $\Hyp$ as is the case  with the aforementionned divergences.
As a result, they can be sensitive to issues such as minor differences of support in $P$ and $Q$, 
which may be irrelevant when learning with certain classes $\Hyp$.

On the algorithmic side, many approaches assign importance weights to source data from $P$ so as to minimize some prescribed \emph{metric} between $P$ and $Q$ \cite{cortes2008sample, gretton2009covariate}; as we will argue, \emph{metrics}, being symmetric, can be inadequate as a measure of discrepancy given the inherent asymmetry in transfer.

The importance of unlabeled data in transfer-learning, given the cost of target labels, has always been recognized, with various approaches developed over the years \citep{huang2007correcting, ben2012hardness}, 
including more recent research efforts into so-called  \emph{semisupervised} or \emph{active} transfer, 
where, given unlabeled target data, the goal is to request as few target labels as possible to improve 
classification over using source data alone \citep{saha2011active, chen2011co, pmlr-v28-chattopadhyay13, yang2013theory, berlind2015active}. 

More recently, \cite{blanchard2017domain, scott2019generalized, cai2019transfer} consider nonparametric transfer settings (unbounded VC) allowing for changes in conditional distributions. Also recent, but more closely related, 
\cite{kpotufe2018marginal} proposed a nonparametric measure of discrepancy which successfully captures the interaction between labeled source and target under nonparametric conditions and 0-1 loss; these notions however ignore the additional structure afforded by transfer in the context of a fixed hypothesis class. 



\section{Setup and Definitions}
We consider a classification setting where the input $X\in \X$, some measurable space, and the output $Y \in \braces{0, 1}$. We let $\Hyp \subset 2^\X$ denote a fixed hypothesis class over $\X$, denote $\V$ the VC dimension \cite{VC:72}, and the goal is to return a classifier $h\in \Hyp$ with low error $R_Q(h) \doteq \expec_{Q} [h(X) \neq Y]$ under some joint distribution $Q$ on $X, Y$. 
The learner has access to two independent labeled samples $S_P\sim P^{n_P}$ and $S_Q \sim Q^{n_Q}$, i.e., drawn from \emph{source} distributions $P$ and target $Q$, of respective sizes $n_P, n_Q$. Our aim is to bound the excess error, under $Q$, of any $\hat h$ learned from both samples, in terms of $n_P, n_Q$, and (suitable) notions of discrepancy between $P$ and $Q$. We will let $P_X, Q_X, P_{Y|X}, Q_{Y|X}$ denote the corresponding marginal and conditional distributions under $P$ and $Q$. 

\begin{definition}
For $D \in \{Q,P\}$, denote $\E_{D}(h) \doteq R_D(h) - \inf_{h'\in \Hyp}R_D(h')$, 
the {\bf excess error} of $h$. 
\end{definition}

\paragraph{Distributional Conditions.}
We consider various traditional assumptions in classification and transfer. The first one is a so-called \emph{Bernstein Class Condition} on noise \cite{bartlett:06,MN:06,koltchinskii:06,tsybakov:04,mammen:99}.

\begin{namedHyp}{\nc}
Let $h^*_P \doteq \argmin\limits_{h \in \Hyp}R_P(h)$ and $h^*_Q \doteq \argmin\limits_{h \in \Hyp}R_Q(h)$ exist. 
$\exists \beta_P, \beta_Q \in [0, 1], c_P, c_Q >0$ s.t. 
\begin{align}
    P_X (h\neq h^*_P) \leq c_p \cdot \E_P^{\beta_P}(h), \quad \text{and} \quad  Q_X (h\neq h^*_Q) \leq c_q \cdot \E_Q^{\beta_Q}(h). \label{eq:bern-class}
\end{align}
\end{namedHyp}
For instance, the usual Tsybakov noise condition, say on $P$, corresponds to the case where 
$h^*_P$ is the Bayes classifier, with corresponding regression function $\eta_P(x) \doteq \expec [Y| x]$ satisfying $P_X(\abs{\eta_P(X) -1/2} \leq \tau) \leq C\tau^{\beta_P/(1-\beta_P)}$. Classification is easiest w.r.t. $P$ (or $Q$) when $\beta_P$ (resp. $\beta_Q$) is largest. We will see that this is also the case in Transfer.

{ The next assumption is stronger, but can be viewed as a relaxed version of the usual \emph{Covariate-Shift} assumption which states that $P_{Y|X} = Q_{Y|X}$. 

\begin{namedHyp}{\rcs}
Let $h^*_P, h^*_Q$ as defined above. We have $\E_Q (h^*_P) = \E_Q(h^*_Q) = 0$. We then define $h^* \doteq h^*_P$.
\end{namedHyp}
Note that the above allows $P_{Y|X} \neq Q_{Y|X}$. However, it is not strictly weaker than \emph{Covariate-Shift}, since the latter allows $h^*_P \neq h^*_Q$ provided the Bayes $\notin \Hyp$. The assumption is useful as it serves to isolate the sources of hardness in transfer beyond just shifts in $h^*$. 
We will in fact see later that it is easily removed, but at the additive (necessary) cost of $\E_Q(h^*_P)$.

\section{Transfer-Exponents from $P$ to $Q$.}
We consider various notions of \emph{discrepancy} between $P$ and $Q$, which will be shown to tightly capture the complexity of transfer $P$ to $Q$. 

\begin{definition}
\label{def:rho}
We call $\rho>0$ a {\bf transfer-exponent} from $P$ to $Q$, w.r.t. $\Hyp$, if there exists $C_\rho$ such that  
\begin{align}
   \forall h \in \Hyp, \quad C_\rho \cdot \E_P(h) \geq \E_Q^\rho(h). 
\end{align}
\end{definition}
We are interested in the smallest such $\rho$ with small $C_\rho$. We generally would think of $\rho$ as at least $1$, although there are situations -- which we refer to as \emph{super-transfer}, to be discussed, where we have $\rho<1$; in such situations, data from $P$ can yield faster $\E_Q$ rates than data from $Q$.

While the transfer-exponent will be seen to tightly capture the two-samples minimax rates of transfer, and can be \emph{adapted to}, practical learning situations call for \emph{marginal} versions that can capture the rates achievable when one has access to unlabeled $Q$ data. 

\begin{definition}
We call $\gamma >0$ a {\bf marginal transfer-exponent} from $P$ to $Q$ if  $\exists C_\gamma$ such that 
\begin{align} 
\forall h \in \Hyp, \quad C_\gamma \cdot P_X(h \neq h^*_P)\geq  Q_X^\gamma(h\neq h^*_P). \label{eq:gamma}
\end{align}
\end{definition}

The following simple proposition relates $\gamma$ to $\rho$. 
\begin{proposition}[From $\gamma$ to $\rho$]
\label{prop:gammatorho}
Suppose Assumptions \nc\, and \rcs\, hold, and that $P$ has marginal transfer-exponent $(\gamma, C_\gamma)$ w.r.t. $Q$. Then $P$ has transfer-exponent  
$\rho \leq \gamma/\beta_P$, where 
$C_\rho = C_\gamma^{\gamma/\beta_P}$.
\end{proposition}
\begin{proof} 
$\forall h \in \Hyp$, we have 
$\E_Q(h) \leq Q_X(h\neq h^*_P) \leq C_\gamma \cdot P_X(h\neq h^*_P)^{1/\gamma}
\leq C_\gamma\cdot \E_P(h)^{\beta_P/\gamma}$.
\end{proof}

\subsection{Examples and Relation to other notions of discrepancy.}
\label{sec:examples}
{\vskip -2mm}In this section, we consider various examples that highlight interesting aspects of $\rho$ and $\gamma$, and their relations to other notions of distance $P\to Q$ considered in the literature. 
Though our results cover noisy cases, in all these examples we assume no noise for simplicity, 
and therefore $\gamma = \rho$. 




{\bf Example 1.} (Non-overlapping supports)
This first example emphasizes the fact that, unlike in much of previous analyses of transfer, the exponents $\gamma, \rho$ do 
not require that $Q_X$ and $P_X$ have overlapping support. This is a welcome property shared also by the $d_{\cal A}$ and $\cal Y$ discrepancy. 

\begin{wrapfigure}{R}{0.15\textwidth}
\vspace{-30pt}
\centering
\includegraphics[width=0.13\textwidth]{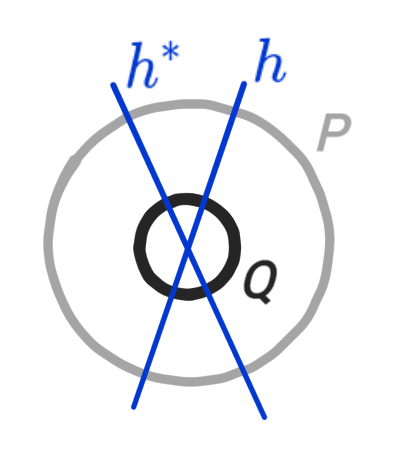}
\end{wrapfigure} 

  In the example shown on the right, $\Hyp$ is the class of homogeneous linear separators, while $P_X$ and $Q_X$ are uniform on the surface of the spheres depicted (e.g., corresponding to different scalings of the data).
We then have that $\gamma = \rho =1$ with $C_\gamma =1$, while notions such as \emph{density-ratios}, KL-divergences, or the recent nonparameteric notion of \cite{kpotufe2018marginal}, are ill-defined or diverge to $\infty$. 


{\bf Example 2.} (Large $d_{\cal A}, d_{\cal Y}$) 
Let $\Hyp$ be the class of one-sided thresholds on the line, and let 
$P_X \doteq \mathcal{U}[0, 2]$ and $Q_X \doteq \mathcal{U}[0, 1]$.
\begin{wrapfigure}{R}{0.3\textwidth}
\vspace{-16pt}
\centering
\includegraphics[width=0.25\textwidth]{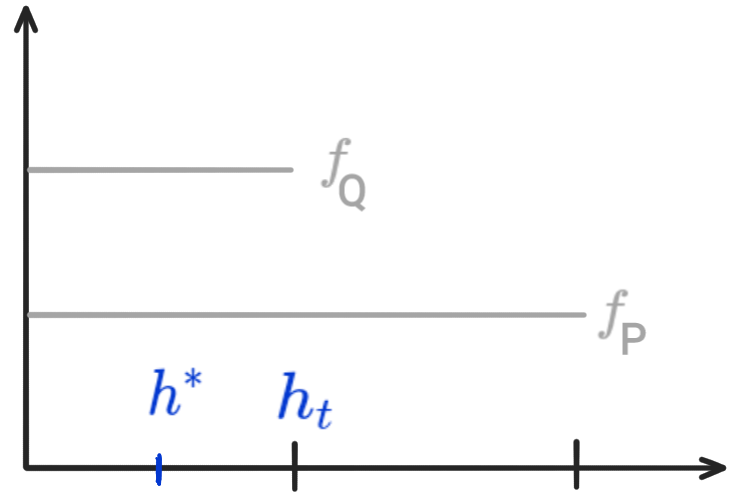}
\vspace{-15pt}
\end{wrapfigure} 
Let $h^*$ be thresholded at $1/2$. We then see that for all $h_t$ thresholded at $t \in [0, 1]$, 
$2P_X(h_t \neq h^*) = Q_X(h_t \neq h^*)$, where for $t>1$, 
$P_X(h_t \neq h^*) = \frac{1}{2}(t-1/2) \geq \frac{1}{2}Q_X(h_t \neq h^*) =  \frac{1}{4}$. 
Thus, the marginal transfer exponent $\gamma = 1$ with $C_\gamma = 2$, so we have fast 
transfer at the same rate $1/n_P$ as if we were sampling from $Q$ {(Theorem \ref{thm:lepski})}.

On the other hand, recall that the $d_{\cal A}$-divergence takes the form 
$d_{\cal A}(P, Q) \doteq \sup_{h \in \Hyp} \abs{P_X(h\neq h^*) - Q_X(h\neq h^*)}$, while the $\cal Y$-discrepancy takes the form $d_{\cal Y}(P, Q) \doteq \sup_{h\in \Hyp} \abs{\E_P(h) - \E_Q(h)}$. The two coincide whenever there is no noise in $Y$.  

Now, take $h_t$ as the threshold at $t = 1/2$, and $d_{\cal A} = d_{\cal Y} =\frac{1}{4}$ which would wrongly imply that transfer is not feasible at a rate faster than $\frac{1}{4}$; we can in fact make this situation worse, i.e., let $d_{\cal A} = d_{\cal Y} \to \frac{1}{2}$ by letting $h^*$ correspond to a threshold close to $0$. A first issue is that these divergences get large in large disagreement regions; this is somewhat mitigated by \emph{localization}, as discussed in Example 4.

{\bf Example 3.} (Minimum $\gamma$, $\rho$, and the inherent {\bf asymmetry} of transfer)
Suppose $\Hyp$ is the class of one-sided thresholds on the line, $h^* = h^*_P = h^*_Q$ is a threshold at $0$. 
\begin{wrapfigure}{R}{0.3\textwidth}
\vspace{-25pt}
\centering
\includegraphics[width=0.25\textwidth]{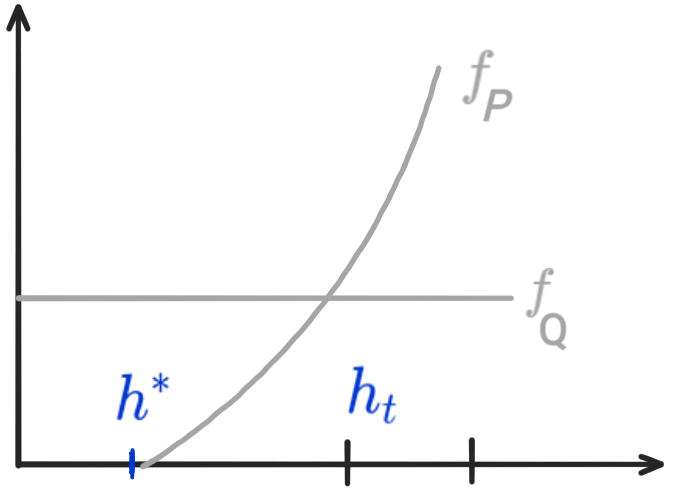}
\vspace{-15pt}
\end{wrapfigure} 
The marginal $Q_X$ has uniform density $f_Q$ (on an interval containing $0$), 
while, for some $\gamma\geq 1$, $P_X$ has 
density $f_P(t) \propto t^{\gamma -1}$ on $t>0$ (and uniform on the rest of the support of $Q$, not shown).  
Consider any $h_t$ at threshold $t>0$, we have $P_X(h_t \neq h^*) = \int_0^t f_P \propto t^\gamma$, 
while $Q_X(h_t \neq h^*) \propto t$. Notice that for any fixed $\epsilon>0$, 
$ \lim\limits_{t>0, \,  t \to 0} \frac{Q_X(h_t \neq h^*)^{\gamma -\epsilon}}{P_X(h_t \neq h^*)} 
= \lim\limits_{t>0, \, t \to 0} C\frac{t^{\gamma -\epsilon}}{t^\gamma} = \infty. $

{\vskip -1mm}We therefore see that $\gamma$ is the smallest possible marginal transfer-exponent (similarly, 
$\rho = \gamma$ is the smallest possible transfer exponent). Interestingly, now consider transferring instead 
from $Q$ to $P$: we would have $\gamma(Q\to P) =1 \leq \gamma \doteq \gamma(P\to Q)$, i.e., it could be easier to transfer from $Q$ to $P$ than from $P$ to $Q$, which is \emph{not captured by symmetric notions of distance}, e.g. metrics ($d_{\cal A}$, $d_{\cal Y}$, MMD, Wassertein, etc ...). 

{\vskip -2mm}Finally note that the above example can be extended to more general 
hypothesis classes as it simply plays on how fast $f_P$ decreases w.r.t. $f_Q$ in regions of space.

{\bf Example 4.} (\emph{Super-transfer} and localization). 
We continue on the above Example 2. Now let $0<\gamma <1$, and let $f_P(t) \propto \abs{t}^{\gamma -1}$ on $[-1, 1]\setminus\{0\}$, 
with $Q_X \doteq \mathcal{U}[-1, 1]$, $h^*$ at $0$. As before, $\gamma$ is a transfer-exponent $P\to Q$, and following from 
{Theorem \ref{thm:lepski}}, we attain transfer rates of $\E_Q \lesssim n_P^{-1/\gamma}$, faster than the rates of $n_Q^{-1}$ attainable with data from $Q$. We call these situations \emph{super-transfer}, i.e., ones where the source data gets us faster to $h^*$;
here $P$ concentrates more mass close to $h^*$, while more generally,  
such situations can also be constructed by letting $P_{Y|X}$ be less noisy than $Q_{Y|X}$ data, for instance corresponding to controlled lab data as source, vs noisy real-world data. 

{\vskip -1mm}Now consider the following $\epsilon$-\emph{localization} fix to the $d_{\cal A}=d_{\cal Y}$ divergences  over $h$'s with small $P$ error (assuming we only observe data from $P$): $d_{\cal Y}^* \doteq \sup_{h\in \Hyp: \, \E_P(h) \leq \epsilon} \abs{\E_P(h) - \E_Q(h)}.$
This is no 
longer worst-case over all $h$'s, yet it is still not a complete fix. To see why, consider that, given $n_P$ data from $P$, the best $P$-excess risk attainable is $n_P^{-1}$ so we might set $\epsilon \propto n_P^{-1}$.  Now the subclass 
$\{h\in \Hyp: \, \E_P(h) \leq \epsilon\}$ corresponds to thresholds $t \in [\pm n_P^{-1/\gamma}]$, since $\E_P(h_t) =P([0, t])\propto \abs{t}^\gamma$. We therefore have $d^*_{\cal Y} \propto \abs{n_P^{-1} - n_P^{-1/\gamma}} \propto n_P^{-1}$, wrongly suggesting a transfer rate 
$\E_Q \lesssim n_P^{-1}$, while the super-transfer rate $n_P^{-1/\gamma}$ is achievable as discussed above. The problem is that, even after localization, $d_{\cal Y}^*$ treats errors under $P$ and $Q$ symmetrically.

%

\section{Lower-Bounds}
\label{sec:lowerboounds}


\begin{definition}[\nc \, Class]
Let $\dist_\nc(\rho, \beta_P, \beta_Q, C)$ denote the class of pairs of distributions 
$(P, Q)$ with transfer-exponent $\rho$, $C_\rho \leq  C$, satisfying \nc\, with parameters 
$\beta_P, \beta_Q$, and $c_P, c_Q \leq C$. 
\end{definition}

The following lower-bound in terms of $\rho$ is obtained via information theoretic-arguments. In effect, given the VC class $\H$, we construct a set of distribution pairs $\{(P_i, Q_i)\}$ supported on $\V$ datapoints, which all belong to $\dist_\nc(\rho, \beta_P, \beta_Q, C)$. 
All the distributions share the same marginals $P_X, Q_X$. 
Any two pairs are close to each other in the sense that $\Pi_i, \Pi_j$, where $\Pi_i \doteq P_i^{n_P} \times Q_i^{n_Q}$, are close in KL-divergence, 
while, however maintaining pairs $(P_i, Q_i), (P_j, Q_j)$ far in a pseudo-distance induced by $Q_X$. 
All the proofs from this section are in Appendix \ref{sec:lowerboundproof}.

\begin{theorem}[$\rho$ Lower-bound]\label{theo:lowrho}
Suppose the hypothesis class $\Hyp$ has VC dimension $\V \geq 9$. 
Let $\hat h = \hat h (S_P, S_Q)$ denote any (possibly improper) classifier with access 
to two independent labeled samples $S_P\sim P^{n_P}$ and $S_Q \sim Q^{n_Q}$. Fix any 
$ \rho\geq 1$, $0\leq \beta_P, \beta_Q < 1$. Suppose either $n_P$ or $n_Q$ is sufficiently large so that 
$$ \epsilon(n_P, n_Q) \doteq  \min\braces{\paren{\frac{\V}{n_P}}^{1/(2-\beta_P)\rho}, \paren{\frac{\V}{n_Q}}^{1/(2- \beta_Q)}} \leq 1/2.$$

Then, for any $\hat h$, there exists $(P, Q) \in \dist_\nc(\rho, \beta_P, \beta_Q, 1)$, and 
a universal constant $c$ such that,
\begin{align*}
    \Prob_{S_P, S_Q}\paren{\E_Q(\hat h) > c\cdot \epsilon(n_P, n_Q) } \geq \frac{3 - 2 \sqrt{2}}{8}.
\end{align*}
\end{theorem}


As per Proposition \ref{prop:gammatorho} we can translate any upper-bound in terms of $\rho$ to an upper-bound in terms of $\gamma$ since $\rho \leq \gamma/\beta_P$. We investigate whether such upper-bounds in terms of $\gamma$ are tight, i.e., given a class $\dist_\nc(\rho, \beta_P, \beta_Q, C)$, are there distributions with $\rho = \gamma/\beta_P$ where the rate is realized.  

The proof of the next result is similar to that of Theorem \ref{theo:lowrho}, however with the added difficulty that we need the construction to yield two forms of rates $\epsilon_1(n_P, n_Q), \epsilon_2(n_P, n_Q)$ over the data support 
(again $\V$ points).
Combining these two rates matches the desired upper-bound. In effect, we follow the intuition that, to have $\rho = \gamma/\beta_P$ achieved on some subset $\X_1\subset \X$, we need $\beta_Q$ to behave as $1$ locally on $\X_1$, while matching the rate requires larger $\beta_Q$ on the rest of the suppport (on $\X \setminus \X_1$).

\begin{theorem}[$\gamma$ Lower-bound]\label{theo:lowgamma}
Suppose the hypothesis class $\Hyp$ has VC dimension $\V,\, \lfloor \V/2\rfloor \geq 9$. 
Let $\hat h = \hat h (S_P, S_Q)$ denote any (possibly improper) classifier with access 
to two independent labeled samples $S_P\sim P^{n_P}$ and $S_Q \sim Q^{n_Q}$. Fix any 
$0<\beta_P, \beta_Q < 1$, $ \rho\geq \max\braces{1/{\beta_P}, 1/\beta_Q}$. Suppose either $n_P$ or $n_Q$ is sufficiently large so that 
\begin{align*}
\epsilon_1(n_P, n_Q) &\doteq  \min\braces{\paren{\frac{\V}{n_P}}^{1/(2-\beta_P)\rho\cdot \beta_Q}, \paren{\frac{\V}{n_Q}}^{1/(2- \beta_Q)}} \leq 1/2, \text{ and } \\
\epsilon_2(n_P, n_Q) &\doteq  \min\braces{\paren{\frac{\V}{n_P}}^{1/(2-\beta_P)\rho}, \paren{\frac{\V}{n_Q}}} \leq 1/2.
\end{align*}

Then, for any $\hat h$, there exists $(P, Q) \in \dist_\nc(\rho, \beta_P, \beta_Q, 2)$, with 
{\bf marginal-transfer-exponent} $\gamma = \rho\cdot \beta_P \geq 1$, with $C_\gamma \leq 2$, 
and a universal constant $c$ such that,
\begin{align*}
  \Expectation_{S_P, S_Q}  \E_Q(\hat h) \geq 
  c\cdot \max \braces{\epsilon_1(n_P, n_Q), \epsilon_2(n_p, n_Q)}.
\end{align*}

\end{theorem}

\begin{remark}[Tightness with upper-bound]
Write $\epsilon_1(n_P, n_Q) = \min \{\epsilon_1(n_P), \epsilon_1(n_Q)\}$, and similarly, $\epsilon_2(n_P, n_Q) = \min \{\epsilon_2(n_P), \epsilon_2(n_Q)\}$. Define $\epsilon_{L} \doteq \max \{\epsilon_1(n_P, n_Q), \epsilon_2(n_P, n_Q)\}$ as in the above lower-bound of Theorem \ref{theo:lowgamma}. Next, define 
$\epsilon_{H} \doteq \min\{\epsilon_2(n_P), \epsilon_1(n_Q)\}$. It turns out that the best upper-bound we can show (as a function of $\gamma$) is in terms of $\epsilon_H$ so defined. It is therefore natural to ask whether or when $\epsilon_H$ and $\epsilon_L$ are of the same order.

Clearly, we have $\epsilon_1(n_P) \leq \epsilon_2(n_P)$ and $\epsilon_1(n_Q) \geq \epsilon_2(n_Q)$ so that $\epsilon_L \leq \epsilon_H$ (as to be expected). 

Now, if $\beta_Q = 1$, we have 
$\epsilon_1(n_P) = \epsilon_2(n_P)$ and $\epsilon_1(n_Q) = \epsilon_2(n_Q)$, so that $\epsilon_L = \epsilon_H$. 
More generally, from the above inequalities, we see that $\epsilon_L = \epsilon_H$ in the two regimes where either 
$\epsilon_1(n_Q) \leq \epsilon_1(n_P)$ (in which case $\epsilon_L = \epsilon_H = \epsilon_1(n_Q)$), or $\epsilon_2(n_P) \leq \epsilon_2(n_Q)$ (in which case $\epsilon_L = \epsilon_H = \epsilon_2(n_P)$). 
\end{remark}

\section{Upper-Bounds}
\label{sec:lepski}

The following lemma is due to \cite{VC:74}.

\begin{lemma}
\label{lem:vc-bernstein}
Let $A_{n} = \frac{\V}{n} \log\paren{ \frac{\max\{n,\V\}}{\V} } + \frac{1}{n} \log\paren{\frac{1}{\delta} }$.
With probability at least $1-\frac{\delta}{3}$, $\forall h,h^{\prime} \in \Hyp$, 
\begin{equation}
\label{eqn:vc-bernstein}
R(h) - R(h^{\prime}) \leq \hat{R}(h) - \hat{R}(h^{\prime}) + c \sqrt{ \min\{ P(h \neq h^{\prime}), \hat{P}( h \neq h^{\prime} ) \} A_{n} } + c A_{n},
\end{equation}
and
\begin{equation}
\label{eqn:vc-empirical-distance-bound}
\frac{1}{2} P( h \neq h^{\prime} ) - c A_{n} \leq \hat{P}( h \neq h^{\prime} ) \leq 2 P( h \neq h^{\prime} ) + c A_{n},
\end{equation}
for a universal numerical constant $c \in (0,\infty)$,
where $\hat{R}$ denotes empirical risk on $n$ iid samples.
\end{lemma}

Now consider the following algorithm.
Let $S_{P}$ be a sequence of $n_{P}$ samples from $P$ and $S_{Q}$ a sequence of $n_{Q}$ samples from $Q$.
Also let $\hat{h}_{S_{P}} = \argmin_{h \in \Hyp} \hat{R}_{S_{P}}(h)$ and $\hat{h}_{S_{Q}} = \argmin_{h \in \Hyp} \hat{R}_{S_{Q}}(h)$.
Choose $\hat{h}$ as the solution to the following optimization problem.

\begin{bigboxit}
\AlgA :
\begin{align}
& \text{Minimize } & & \hat{R}_{S_{P}}(h) \nonumber
\\ & \text{subject to } & & \hat{R}_{S_{Q}}(h) - \hat{R}_{S_{Q}}(\hat{h}_{S_{Q}}) \leq c \sqrt{ \hat{P}_{S_{Q}}( h \neq \hat{h}_{S_{Q}} ) A_{n_{Q}} } + c A_{n_{Q}} \label{eq:lepski}
\\ & & & h \in \Hyp \nonumber.
\end{align}
\end{bigboxit}

The intuition is that, effectively, the constraint guarantees we maintain a 
near-optimal guarantee on $\E_{Q}(\hat{h})$ in terms of $n_{Q}$ and the \nc\, parameters for $Q$,
while (as we show) still allowing the algorithm to select an $h$ 
with a near-minimal value of $\hat{R}_{S_{P}}(h)$.
The latter guarantee plugs into the transfer condition to obtain a term converging 
in $n_{P}$, while the former provides a term converging in $n_{Q}$, 
and altogether the procedure achieves a rate specified by the \emph{min} of these two guarantees
(which is in fact nearly minimax \emph{optimal}, since it matches the lower bound up to logarithmic factors).

Formally, we have the following result for this learning rule;  
its proof is below. 

\begin{theorem}[Minimax Upper-Bounds]
\label{thm:lepski}
Assume \nc. Let $\hat{h}$ be the solution from \AlgA.
For a constant $C$ depending on $\rho,C_{\rho},\beta_{P},c_{\beta_{P}},\beta_{Q},c_{\beta_{Q}}$, 
with probability at least $1-\delta$, 
\begin{equation*}
\E_{Q}(\hat{h}) \leq C \min\!\left\{ A_{n_{P}}^{\frac{1}{(2-\beta_{P})\rho}}, 
A_{n_{Q}}^{\frac{1}{2-\beta_{Q}}} \right\}
= \tilde{O}\!\left( \min\!\left\{ 
\left(\frac{\V}{n_{P}}\right)^{\frac{1}{(2-\beta_{P})\rho}},
\left(\frac{\V}{n_{Q}}\right)^{\frac{1}{2-\beta_{Q}}}
\right\}\right).
\end{equation*}
\end{theorem}

Note that, by the lower bound of Theorem \ref{theo:lowrho}, 
this bound is optimal up to log factors.

\begin{remark}[Effective Source Sample Size] From the above, we might view (ignoring $\V$) $\tilde n_P \doteq n_P^{(2-\beta_Q)/(2-\beta_P)\rho}$ as the effective sample size contributed by $P$. In fact, the above minimax rate is of order $(\tilde n_P + n_Q)^{-1/(2-\beta_Q)}$, which yields added intuition into the combined effect of both samples. We have that, the effective source sample size $\tilde n_P$ is smallest for large $\rho$, but also depends on $(2-\beta_Q)/(2-\beta_P)$, i.e., on whether $P$ is noisier than $Q$.

\end{remark}

\begin{remark}[Rate in terms of $\gamma$] 
Note that, by Proposition~\ref{prop:gammatorho}, 
this also immediately implies a bound 
under the marginal transfer condition and RCS, 
simply taking $\rho \leq \gamma/\beta_{P}$.
Furthermore, by the lower bound of Theorem \ref{theo:lowgamma}, 
the resulting bound in terms of $\gamma$ is tight in certain regimes up to 
log factors.  
\end{remark}


\begin{proof}[Proof of Theorem~\ref{thm:lepski}]
In all the lines below, we let $C$ serve as a generic constant (possibly depending on $\rho,C_{\rho},\beta_{P},c_{\beta_{P}},\beta_{Q},c_{\beta_{Q}}$) which may be different in different appearances.
Consider the event of probability at least $1-\delta/3$ from Lemma~\ref{lem:vc-bernstein} for the $S_{Q}$ samples.
In particular, on this event, if $\E_Q(h^*_P)=0$, it holds that 
\begin{equation*}
\hat{R}_{S_{Q}}(h^{*}_P) - \hat{R}_{S_{Q}}(\hat{h}_{S_{Q}}) 
\leq c \sqrt{\hat{P}_{S_{Q}}( h^{*}_P \neq \hat{h}_{S_{Q}} ) A_{n_{Q}}} + c A_{n_{Q}}.
\end{equation*}
This means, under the \rcs\ condition, 
$h^{*}_P$ satisfies the constraint in the above optimization problem defining $\hat{h}$.
Also, 
on this same event from Lemma~\ref{lem:vc-bernstein} we have 
\begin{equation*}
\E_{Q}(\hat{h}_{S_{Q}}) 
\leq c \sqrt{ Q( \hat{h}_{S_{Q}} \neq h^{*}_Q ) A_{n_{Q}}} + c A_{n_{Q}},
\end{equation*}
so that \nc\ implies 
\begin{equation*}
\E_{Q}(\hat{h}_{S_{Q}}) 
\leq C \sqrt{ \E_{Q}(\hat{h}_{S_{Q}})^{\beta_{Q}} A_{n_{Q}}} + c A_{n_{Q}},
\end{equation*}
which implies the well-known fact from \cite{MN:06,koltchinskii:06} that
\begin{equation}
\label{eqn:EP-hatP-bound}
\E_{Q}(\hat{h}_{S_{Q}}) \leq C \left( \frac{\V}{n_{Q}} \log\!\left( \frac{n_{Q}}{\V} \right) + \frac{1}{n_{Q}} \log\!\left( \frac{1}{\delta} \right) \right)^{\frac{1}{2-\beta_{Q}}}.
\end{equation}
%
%
Furthermore, following the analogous argument for $S_{P}$, it follows that 
for any set $\Gyp \subseteq \Hyp$ with $h^{*}_{P} \in \Gyp$, with probability at least $1-\delta/3$, the ERM $\hat{h}_{S_{P}}^{\prime} = \argmin_{h \in \Gyp} \hat{R}_{S_{P}}(h)$ satisfies
\begin{equation}
\label{eqn:Gyp-eqn}
\E_{P}(\hat{h}_{S_{P}}^{\prime}) \leq C \left( \frac{\V}{n_{P}} \log\!\left( \frac{n_{P}}{\V} \right) + \frac{1}{n_{P}} \log\!\left( \frac{1}{\delta} \right) \right)^{\frac{1}{2-\beta_{P}}}. 
\end{equation}
In particular, conditioned on the $S_{Q}$ data, we can take the set $\Gyp$ as the set of $h \in \Hyp$ satisfying the constraint in the optimization, 
and on the above event we have $h^{*}_P \in \Gyp$ (assuming the \rcs\ condition); 
furthermore, if $\E_Q(h^*_P) = 0$, then without loss 
we can simply define $h^*_Q = h^*_P = h^*$ 
(and it is easy to see that this does not affect the NC condition). 
We thereby 
establish the above inequality \eqref{eqn:Gyp-eqn} for this choice of $\Gyp$, 
in which case by definition $\hat{h}_{S_{P}}^{\prime} = \hat{h}$.
Altogether, 
by the union bound, all of these events hold simultaneously with probability at least $1-\delta$.
In particular, on this event, if the \rcs\ condition holds then  
\begin{equation*}
\E_{P}(\hat{h}) \leq C \left( \frac{\V}{n_{P}} \log\!\left( \frac{n_{P}}{\V} \right) + \frac{1}{n_{P}} \log\!\left( \frac{1}{\delta} \right) \right)^{\frac{1}{2-\beta_{P}}}.
\end{equation*}
Applying the definition of $\rho$, this has the 
further implication that (again if \rcs\ holds)
\begin{equation*}
\E_{Q}(\hat{h}) \leq C \left( \frac{\V}{n_{P}} \log\!\left( \frac{n_{P}}{\V} \right) + \frac{1}{n_{P}} \log\!\left( \frac{1}{\delta} \right) \right)^{\frac{1}{(2-\beta_{P})\rho}}.
\end{equation*}
Also note that, if $\rho = \infty$ this inequality trivially holds, 
whereas if $\rho < \infty$ then \rcs\ necessarily holds so that 
the above implication is generally valid, without needing the 
\rcs\ assumption explicitly.
Moreover, again when the above events hold, using the event from Lemma~\ref{lem:vc-bernstein} again,
along with the constraint from the optimization, we have that 
\begin{equation*}
R_{Q}(\hat{h}) - R_{Q}(\hat{h}_{S_{Q}}) 
\leq 2 c \sqrt{ \hat{P}_{S_{Q}}( \hat{h} \neq \hat{h}_{S_{Q}} ) A_{n_{Q}} } + 2 c A_{n_{Q}},
\end{equation*}
and \eqref{eqn:vc-empirical-distance-bound} implies the right hand side is at most 
\begin{equation*}
C \sqrt{ Q( \hat{h} \neq \hat{h}_{S_{Q}} ) A_{n_{Q}} } + C A_{n_{Q}}
\leq C \sqrt{ Q( \hat{h} \neq h^{*}_Q) A_{n_{Q}} }+ C \sqrt{ Q( \hat{h}_{S_{Q}} \neq h^*_Q ) A_{n_{Q}} } + C A_{n_{Q}}.
\end{equation*}
Using the Bernstein class condition and \eqref{eqn:EP-hatP-bound}, the second term is bounded by 
\begin{equation*}
C \left( \frac{\V}{n_{Q}} \log\!\left( \frac{n_{Q}}{\V} \right) + \frac{1}{n_{Q}} \log\!\left( \frac{1}{\delta} \right) \right)^{\frac{1}{2-\beta_{Q}}},
\end{equation*}
while the first term is bounded by 
\begin{equation*}
C \sqrt{ \E_{Q}(\hat{h})^{\beta_{Q}} A_{n_{Q}} }.
\end{equation*}
Altogether, we have that
\begin{align*}
 \E_{Q}(\hat{h}) &= R_{Q}(\hat{h}) - R_{Q}(\hat{h}_{S_{Q}}) + \E_{Q}(\hat{h}_{S_{Q}}) 
\\ & \leq C \sqrt{ \E_{Q}(\hat{h})^{\beta_{Q}} A_{n_{Q}} } + C \left( \frac{\V}{n_{Q}} \log\!\left( \frac{n_{Q}}{\V} \right) + \frac{1}{n_{Q}} \log\!\left( \frac{1}{\delta} \right) \right)^{\frac{1}{2-\beta_{Q}}},
\end{align*}
which implies 
\begin{equation*}
\E_{Q}(\hat{h}) \leq C \left( \frac{\V}{n_{Q}} \log\!\left( \frac{n_{Q}}{\V} \right) + \frac{1}{n_{Q}} \log\!\left( \frac{1}{\delta} \right) \right)^{\frac{1}{2-\beta_{Q}}}.
\end{equation*}
{\vskip -6mm}
\end{proof}

\begin{remark} 
Note that the above Theorem \ref{thm:lepski} does not require \rcs: 
that is, it holds even when $\E_Q(h^*_P) > 0$, in which case $\rho = \infty$.
However, for a related method we can also show a stronger result 
in terms of a modified definition of $\rho$:
\end{remark} 

{\vskip -2mm}Specifically, define 
$\E_Q(h,h_P^*) = \max\{R_Q(h) - R_Q(h_P^*),0\}$, 
and suppose $\rho^{\prime} > 0$, $C_{\rho^{\prime}} > 0$ satisfy 
\begin{equation*}
\forall h \in \Hyp,~~~ C_{\rho^{\prime}} \cdot \E_P(h) \geq \E_Q^{\rho^{\prime}}(h,h_P^*). 
\end{equation*}
This is clearly equivalent to $\rho$ (Definition~\ref{def:rho}) under \rcs;
however, unlike $\rho$, this $\rho^{\prime}$ can be finite even in cases where \rcs\ fails.
With this definition, we have the following result.
\begin{proposition}[Beyond \rcs] 
\label{prop:lepski-agnostic}
If $\hat{R}_{S_{Q}}(\hat{h}_{S_{P}}) - \hat{R}_{S_{Q}}(\hat{h}_{S_{Q}}) \leq c \sqrt{ \hat{P}_{S_{Q}}( \hat{h}_{S_{P}} \neq \hat{h}_{S_{Q}} ) A_{n_{Q}} } + c A_{n_{Q}}$, that is, if $\hat{h}_{S_{P}}$ satisfies \eqref{eq:lepski}, 
define $\hat{h} = \hat{h}_{S_{P}}$, and otherwise define $\hat{h} = \hat{h}_{S_{Q}}$.
Assume \nc.
For a constant $C$ depending on $\rho^{\prime},C_{\rho^{\prime}},\beta_{P},c_{\beta_{P}},\beta_{Q},c_{\beta_{Q}}$, 
with probability at least $1-\delta$, 
\begin{equation*}
\E_{Q}(\hat{h}) \leq \min\!\left\{ \E_{Q}(h^{*}_{P}) +  C A_{n_{P}}^{\frac{1}{(2-\beta_{P})\rho^{\prime}}}, C A_{n_{Q}}^{\frac{1}{2-\beta_{Q}}} \right\}.
\end{equation*}
\end{proposition}
The proof of this result is similar to that of Theorem~\ref{thm:lepski}, 
and as such is deferred 
to Appendix~\ref{app:lepski}.

\paragraph{An alternative procedure.} Similar results as in Theorem \ref{thm:lepski} can 
be obtained for a method that swaps 
the roles of $P$ and $Q$ samples:

\begin{bigboxit}
\AlgA$^\prime$ :
\begin{align*}
& \text{Minimize } & & \hat{R}_{S_{Q}}(h) 
\\ & \text{subject to } & & \hat{R}_{S_{P}}(h) - \hat{R}_{S_{P}}(\hat{h}_{S_{P}}) \leq c \sqrt{ \hat{P}_{S_{P}}( h \neq \hat{h}_{S_{P}} ) A_{n_{P}} } + c A_{n_{P}}
\\ & & & h \in \Hyp .
\end{align*}
\end{bigboxit}

This version, more akin to so-called \emph{hypothesis transfer} may have practical benefits 
in scenarios where 
the $P$ data is accessible \emph{before} the $Q$ data, 
since then the feasible set might be calculated (or approximated) 
in advance, so that the $P$ data itself would no longer be needed in order 
to execute the procedure. However this procedure presumes that $h^*_P$ is not far from $h^*_Q$, i.e., that data $S_P$  from $P$ is not misleading, since it conditions on doing well on $S_P$. Hence we now require \rcs. 
\begin{proposition} Assume \nc\, and \rcs. Let $\hat{h}$ be the solution from \AlgA$^{\prime}$.
For a constant $C$ depending on $\rho,C_{\rho},\beta_{P},c_{\beta_{P}},\beta_{Q},c_{\beta_{Q}}$, 
with probability at least $1-\delta$, 
\begin{equation*}
\E_{Q}(\hat{h}) \leq C \min\!\left\{ A_{n_{P}}^{\frac{1}{(2-\beta_{P})\rho}}, 
A_{n_{Q}}^{\frac{1}{2-\beta_{Q}}} \right\}
= \tilde{O}\!\left( \min\!\left\{ 
\left(\frac{\V}{n_{P}}\right)^{\frac{1}{(2-\beta_{P})\rho}},
\left(\frac{\V}{n_{Q}}\right)^{\frac{1}{2-\beta_{Q}}}
\right\}\right).
\end{equation*}
\end{proposition}
The proof is very similar to that of Theorem~\ref{thm:lepski}, 
so is omitted for brevity.





\section{Minimizing Sampling Cost}
\label{sec:adaptive}

In this section (and continued in Appendix~\ref{sec:reweighting}), 
we discuss the value of having access to \emph{unlabeled} data from $Q$.
The idea is that unlabeled data can be obtained much more cheaply than labeled data, so gaining access to unlabeled data can be realistic in many applications.
Specifically, we begin by discussing an adaptive sampling scenario, where we are able to \emph{draw} samples from $P$ or $Q$, 
at different \emph{costs}, and we are interested in optimizing the total cost of obtaining a given excess $Q$-risk.


Formally, consider the scenario where we have as input a value $\epsilon$, 
and are tasked with producing a classifier $\hat{h}$ with $\E_{Q}(\hat{h}) \leq \epsilon$.
We are then allowed to \emph{draw} samples from either $P$ or $Q$ toward achieving this goal, but at different costs. 
Suppose $\cost_{P} : \nats \to [0,\infty)$ and $\cost_{Q} : \nats \to [0,\infty)$ 
are \emph{cost functions}, where $\cost_{P}(n)$ indicates the cost of sampling a batch of size $n$ from $P$, 
and similarly define $\cost_{Q}(n)$. 
We suppose these functions are increasing, and concave, and unbounded.

\begin{definition} 
Define $n_{Q}^{*} = {\V}/{\epsilon^{2-\beta_{Q}}}$,
$n_{P}^{*} = {\V}/{\epsilon^{(2-\beta_{P})\gamma/\beta_{P}}} $,
and 
$\cost^{*} = \min\!\left\{ \cost_{Q}(n_{Q}^{*}), \cost_{P}(n_{P}^{*}) \right \}$. We call 
$\cost^{*} = \cost^*(\epsilon; \cost_P, \cost_Q)$ the  {\bf minimax optimal cost}  of sampling from $P$ or $Q$ to attain $Q$-error $\epsilon$.
\end{definition}

Note that the cost $\cost^{*}$ is effectively the smallest possible, up to log factors, 
in the range of parameters given in Theorem~\ref{theo:lowgamma}. 
That is, in order to make the lower bound in Theorem~\ref{theo:lowgamma} 
less than $\epsilon$, 
either $n_{Q} = \tilde{\Omega}(n_{Q}^{*})$ samples are needed from $Q$ or $n_{P} = \tilde{\Omega}(n_{P}^{*})$ samples are needed from $P$. We show that $\cost^{*}$ is nearly achievable, adaptively with no knowledge of distributional parameters. 

\paragraph{Procedure.}
We assume access to a large unlabeled data set $U_{Q}$ sampled from $Q_X$.
For our purposes, we will suppose this data set has size at least $\Theta( \frac{\V}{\epsilon} \log \frac{1}{\epsilon} + \frac{1}{\epsilon} \log \frac{1}{\delta} )$.

Let $A_{n}^{\prime} = \frac{\V}{n} \log(\frac{\max\{n,\V\}}{\V}) + \frac{1}{n} \log( \frac{2n^{2}}{\delta} )$.  
Then for any labeled data set $S$, define $\hat{h}_{S} = \argmin_{h \in \Hyp} \hat{R}_{S}(h)$, and given an additional data set $U$ (labeled or unlabeled) define a quantity
\begin{equation*}
\hat{\delta}(S,U) = \sup\!\left\{ \hat{P}_{U}( h \neq \hat{h}_{S} ) : h \in \Hyp, \hat{R}_{S}(h) - \hat{R}_{S}(\hat{h}_{S}) \leq c \sqrt{\hat{P}_{S}( h \neq \hat{h}_{S} ) A_{|S|}^{\prime}} + c A_{|S|}^{\prime} \right\},
\end{equation*}
where $c$ is as in Lemma~\ref{lem:vc-bernstein}.
Now we have the following procedure.

\begin{bigboxit}
\AlgB :\\
0. $S_{P} \gets \{\}$, $S_{Q} \gets \{\}$\\
1. For $t = 1,2,\ldots$\\
2.\quad Let $n_{t,P}$ be minimal such that $\cost_{P}(n_{t,P}) \geq 2^{t-1}$\\
3.\quad Sample $n_{t,P}$ samples from $P$ and add them to $S_{P}$\\
4.\quad Let $n_{t,Q}$ be minimal such that $\cost_{Q}(n_{t,Q}) \geq 2^{t-1}$\\
5.\quad Sample $n_{t,Q}$ samples from $Q$ and add them to $S_{Q}$\\
6.\quad If $c \sqrt{ \hat{\delta}(S_{Q},S_{Q}) A_{|S_{Q}|} } + c A_{|S_{Q}|} \leq \epsilon$, return $\hat{h}_{S_{Q}}$\\
7.\quad If $\hat{\delta}(S_{P},U_{Q}) \leq \epsilon/4$, return $\hat{h}_{S_{P}}$
\end{bigboxit}

The following theorem asserts that this procedure will find a classifier $\hat{h}$ with $\E_{Q}(\hat{h}) \leq \epsilon$ 
while adaptively using a near-minimal cost associated with achieving this.
The proof is in Appendix~\ref{app:adaptive}.

\begin{theorem}
[Adapting to Sampling Costs]\label{thm:adaptive-cost}
Assume \nc\, and \rcs. There exist a constant $c'$, depending on parameters ($C_{\gamma}$, $\gamma$, $c_{\beta_{Q}}$, $\beta_{Q}$, $c_{\beta_{P}}$, $\beta_{P}$) but not on $\epsilon$ or $\delta$, such that the following holds. Define sample sizes $\tilde n_{Q} = \frac{c^{\prime}}{\epsilon^{2-\beta_{Q}}} \left( \V \log \frac{1}{\epsilon} + \log \frac{1}{\delta} \right)$, and 
$\tilde n_{P} = \frac{c^{\prime}}{\epsilon^{(2-\beta_{P})\gamma/\beta_{P}}} \left( \V \log \frac{1}{\epsilon} + \log \frac{1}{\delta} \right)$.

\AlgB\ outputs a classifier $\hat{h}$ such that, with probability at least $1-\delta$, 
we have $\E_{Q}(\hat{h}) \leq \epsilon$, 
and the total sampling cost incurred is at most
$\min\!\left\{ \cost_{Q}(\tilde n_{Q}), \cost_{P}(\tilde n_{P})\right\} = \tilde O(\cost^*)$. 
\end{theorem}


Thus, when $\cost^*$ favors sampling from $P$, 
we end up sampling very few labeled $Q$ data. 
These are scenarios where $P$ samples are cheap 
relative to the 
cost of $Q$ samples and w.r.t.\ parameters 
($\beta_{Q},\! \beta_{P},\! \gamma$) which determine the effective source sample size contributed for every target sample. 
Furthermore, we achieve this adaptively: without knowing 
(or even estimating) 
these relevant parameters.
%



\subsection*{Acknowledgments}

We thank Mehryar Mohri for several very important discussions which 
helped crystallize many essential questions and directions on this topic.

\bibliographystyle{unsrt}  
\bibliography{refs}


\newpage
\appendix

\section{Additional Results}
\subsection{Reweighting the Source Data}
\label{sec:reweighting}

In this section, we present a technique for using unlabeled data from $Q$ 
to find a reweighting of the $P$ data more suitable for transfer.
This gives a technique for using the data effectively in a potentially practical way.
As above, we again suppose access to the sample $U_{Q}$ of unlabeled data from $Q$. 

Additionally, we suppose we have access to a set $\Pclass$ of functions $f : \X \to [0,\infty)$, 
which we interpret as unnormalized \emph{density} functions with respect to $P_{X}$.
Let $P_{f}$ denote the bounded measure whose marginal on $\X$ has density $f$ with respect to $P_{X}$, and the conditional $Y|X$ is the same as for $P$.

Now suppose $S_{P} = \{ (x_{i},y_{i}) \}_{i=1}^{n_{P}}$ is a sequence of $n_{P}$ iid $P$-distributed samples.
Continuing conventions from above $R_{P_{f}}(h) = \int \ind[ h(x) \neq y ] f(x) {\rm d}P(x,y)$ is a risk with respect to $P_{f}$,
but now we also write $\hat{R}_{S_{P},f}(h) = \frac{1}{n_{P}} \sum_{(x,y) \in S_{P}} \ind[ h(x) \neq y ] f(x)$, 
and additionally 
we will use $P_{f^{2}}( h \neq h' ) = \int \ind[ h(x) \neq h'(x) ] f^{2}(x) {\rm d}P(x,y)$, 
and $\hat{P}_{S_{P},f^{2}}( h \neq h' ) = \frac{1}{n_{P}} \sum_{(x,y) \in S_{P}} \ind[ h(x) \neq h'(x) ] f^{2}(x)$; 
the reason $f^2$ is used instead of $f$ is that this will represent a variance term in the bounds below.
Other notations from above are defined analogously.
In particular, also let $\hat{h}_{S_{P},f} = \argmin_{h \in \Hyp} \hat{R}_{S_{P},f}(h)$.
For simplicity, we will only present the case of $\Pclass$ having finite pseudo-dimension $\Pdim$ 
(i.e., $\Pdim$ is the VC dimension of the subgraph functions $\{(x,y) \mapsto \ind[ f(x) \leq y ] : f \in \Pclass\}$); 
extensions to general bracketing or empirical covering follow similarly.

For the remaining results in this section, we suppose the condition RCS holds for \emph{all} $P_{f}$: 
that is, $R_{P_{f}}$ is minimized in $\Hyp$ at a function $h^{*}_{P_{f}}$ having $\E_{Q}(h^{*}_{P_{f}}) = 0$.
For instance, this would be the case if the Bayes optimal classifier is in the class $\Hyp$.

Define $A_{n}^{\prime\prime} = \frac{\V+\Pdim}{n} \log\!\left( \frac{\max\{n,\V+\Pdim\}}{\V+\Pdim} \right) + \frac{1}{n} \log\!\left( \frac{1}{\delta} \right)$.
Let us also extend the definition of $\hat{\delta}$ introduced above.
Specifically, define $\hat{\delta}(S_{P},f,U_{Q})$ as 
\begin{equation*}
 \sup\!\left\{ \hat{P}_{U_{Q}}( h \neq \hat{h}_{S_{P},f} ) : h \in \Hyp, \hat{\E}_{S_{P},f}(h) \leq c \sqrt{ \hat{P}_{S_{P},f^{2}}( h \neq \hat{h}_{S_{P},f} ) A_{n_{P}}^{\prime\prime} } + c \|f\|_{\infty} A_{n_{P}}^{\prime\prime}  \right\}.
\end{equation*}
Now consider the following procedure.

\begin{bigboxit}
\AlgC :\\
Choose $\hat{f}$ to minimize $\hat{\delta}(S_{P},f,U_{Q})$ over $f \in \Pclass$.
\\Choose $\hat{h}$ to minimize $\hat{R}_{S_{Q}}(h)$ among $h \in \Hyp$ 
\\subject to $\hat{\E}_{S_{P},\hat{f}}(h) \leq c \sqrt{ \hat{P}_{S_{P},\hat{f}^{2}}( h \neq \hat{h}_{S_{P},\hat{f}} ) A_{n_{P}}^{\prime\prime} } + c \| \hat{f} \|_{\infty} A_{n_{P}}^{\prime\prime}$.
\end{bigboxit}

As we establish in the proof, 
$\hat{f}$ is effectively being chosen to minimize an upper bound on the excess $Q$-risk of the resulting classifier $\hat{h}$.
Toward analyzing the performance of this procedure, 
note that each $f$ induces a marginal transfer exponent: 
that is, values $C_{\gamma,f}$, $\gamma_{f}$ such that 
$\forall h \in \H$, 
$
C_{\gamma,f} P_{f^{2}}( h \neq h_{P_{f}}^{*} ) \geq Q^{\gamma_{f}}( h \neq h_{P_{f}}^{*} ).
$
Similarly, each $f$ induces a Bernstein Class Condition: there exist values $c_{f} > 0$, $\beta_{f} \in [0,1]$ such that 
$
P_{f^{2}}( h \neq h_{P_{f}}^{*} ) \leq c_{f} \E_{P_{f}}^{\beta_{f}}(h).
$

The following theorem reveals that \AlgC\ is able to perform nearly as well as
applying the transfer technique from Theorem~\ref{thm:lepski} directly under 
the measure in the family $\Pclass$ that would provide the best bound.
The only losses compared to doing so are a dependence on $\Pdim$ 
and 
the supremum of the density 
(which accounts for how different that measure is from $P$).
The proof is in Appendix~\ref{app:reweighting}.

\begin{theorem}
\label{thm:reweighting-finite}
Suppose $\beta_{Q} > 0$ and that \nc\, and \rcs\, hold for all $P_{f}$, $f \in \Pclass$.
There exist constants $C_{f}$ depending on $\|f\|_{\infty}$, $C_{\gamma,f}$, $\gamma_{f}$, $c_{f}$, $\beta_{f}$, 
and a constant $C$ depending on $c_{q}$, $\beta_{Q}$ 
such that, for a sufficiently large $|U_{Q}|$, 
w.p. at least $1-\delta$, 
the classifier $\hat{h}$ chosen by \AlgC\ satisfies 
\begin{align*}
\E_{Q}(\hat{h}) 
& \leq \inf_{f \in \Pclass} C \min\!\left\{ C_{f} \left( A_{n_{P}}^{\prime\prime} \right)^{\frac{\beta_{f}}{(2-\beta_{f})\gamma_{f}}}, A_{n_{Q}}^{\frac{1}{2-\beta_{Q}}} \right\} 
\\ & = \tilde{O}\!\left( \inf_{f \in \Pclass} \min\!\left\{ C_{f} \left(\frac{\V + \Pdim}{n_{P}}\right)^{\frac{\beta_{f}}{(2-\beta_{f})\gamma_{f}}}, \left(\frac{\V}{n_{Q}}\right)^{\frac{1}{2-\beta_{Q}}} \right\} \right).
\end{align*}
\end{theorem}

The utility of this theorem will of course depend largely on the family $\Pclass$ 
of densities.  This class should contain a distribution with small $\gamma_{f}$ 
marginal transfer exponent, while also small $\|f\|_{\infty}$ (which is captured by the $C_f$ constant in the bound), 
and favorable noise conditions (i.e., large $\beta_{f}$).


\subsection{Choice of Transfer from Multiple Sources}
\label{sec:P1vsP2}

It is worth noting that all of the above analysis also applies 
to the case that, instead of a family of densities with respect to a single $P$, 
the set $\Pclass$ is a set of probability measures $P_{i}$, each with its own separate iid data set $S_{i}$ 
of some size $n_{i}$.
Lemma~\ref{lem:vc-bernstein} can then be applied to all of these data sets, 
if we simply replace $\delta$ by $\delta/|\Pclass|$ to accommodate a union bound; 
call the corresponding quantity $A_{n}^{'''}$.
Then, similarly to the above, we can use the following procedure.

\begin{bigboxit}
\AlgD :\\
Choose $\hat{i}$ to minimize $\hat{\delta}(S_{i},U_{Q})$ over $P_{i} \in \Pclass$.
\\Choose $\hat{h}$ to minimize $\hat{R}_{S_{Q}}(h)$ among $h \in \Hyp$ 
\\subject to $\hat{\E}_{S_{\hat{i}}}(h) \leq c \sqrt{ \hat{P}_{S_{\hat{i}}}( h \neq \hat{h}_{S_{\hat{i}}} ) A_{n_{\hat{i}}}^{'''} } + c A_{n_{\hat{i}}}^{'''}$.
\end{bigboxit}

To state a formal guarantee, let us suppose the conditions above hold for each 
of these distributions with respective values of $C_{\gamma,i}$, $\gamma_{i}$, $c_{i}$, $\beta_{i}$.
We have the following theorem.  
Its proof is essentially 
identical to the proof of Theorem~\ref{thm:reweighting-finite} (effectively just substituting notation),
and is therefore omitted.

\begin{theorem}
\label{thm:P1vsP2}
Suppose $\beta_{Q} > 0$ and that \nc\, and \rcs\, hold for all $P_{i} \in \Pclass$.
There exist constants $C_{i}$ depending on $C_{\gamma,i}$, $\gamma_{i}$, $c_{i}$, $\beta_{i}$, 
and a constant $C$ depending on $c_{q}$, $\beta_{Q}$ 
such that, for a sufficiently large $|U_{Q}|$, 
with probability at least $1-\delta$, 
the classifier $\hat{h}$ chosen by \AlgD\ satisfies 
\begin{equation*}
\E_{Q}(\hat{h}) 
\leq \tilde{O}\!\left( \inf_{P_{i} \in \Pclass} \min\!\left\{ C_{i} \left(\frac{\V + \log(|\Pclass|)}{n_{i}}\right)^{\frac{\beta_{i}}{(2-\beta_{i})\gamma_{i}}}, \left(\frac{\V}{n_{Q}}\right)^{\frac{1}{2-\beta_{Q}}} \right\} \right).
\end{equation*}
\end{theorem}

\section{Lower-Bounds Proofs}
\label{sec:lowerboundproof}
Our lower-bounds rely on the following extensions of Fano inequality.

\begin{proposition} [Thm 2.5 of \cite{tsybakov2009introduction}] \label{prop:tsy25} Let $\{ \Pi_{h} \}_{h \in \Hyp}$ be a family of distributions indexed over a subset $\Hyp$ of a semi-metric $( \mathcal{F}, \semiMetric)$. Suppose $\exists \, h_0, \ldots, h_{M} \in \Hyp$, where $M \geq 2$, such that:
\begin{flalign*} 
\qquad {\rm (i)} \quad  &\semiDist{h_{i}}{h_{j}} \geq 2 s > 0, \quad \forall 0 \leq i < j \leq M,  & \\
\qquad {\rm (ii)} \quad  & \Pi_{h_i} \ll \Pi_{h_0} \quad \forall i \in  [M], \text{ and the average  KL-divergence to } \Pi_{h_0} \text{ satisfies } & \\
& \qquad 
\frac{1}{M} \sum_{i = 1}^{M} \KLDiv{\Pi_{h_i}}{ \Pi_{h_0}} \leq \alpha \log M, \text{ where } 0 < \alpha < 1/8.
\end{flalign*}
Let $Z\sim\Pi_{h}$, and let $\hat h : Z \mapsto \mathcal{F}$ denote any \emph{improper} learner of $h\in \Hyp$. We have for any $\hat h$: 
\begin{equation*}
\sup_{h \in \Hyp} \Pi_{h} \left( \semiDist{\hat h(Z)}{h} \geq s \right) \geq \frac{\sqrt{M}}{1 + \sqrt{M}} \left( 1 - 2 \alpha - \sqrt{\frac{2 \alpha}{\log(M)}} \right) \geq \frac{3 - 2 \sqrt{2}}{8}.
\end{equation*}
\end{proposition}

The following proposition would be needed to construct packings (of spaces of distributions) of the appropriate size. 

\begin{proposition} [Varshamov-Gilbert bound] \label{lem:VGBound}
Let $d \geq 8$. Then there exists a subset $\{ \sigma_0, \ldots, \sigma_{M}\}$ of $\{-1 ,1 \}^{d}$ such that $\sigma_0 = (1,\ldots,1)$,
\begin{equation*}
\text{dist}(\sigma_{i},\sigma_{j}) \geq \frac{d}{8}, \quad \forall\,  0 \leq i < j \leq M, \quad \text{and} \quad M \geq 2^{d / 8},
\end{equation*}
where $\text{dist}(\sigma,\sigma') \doteq \text{card}(\{ i \in [m] :  \sigma(i) \neq \sigma'(i) \})$ is the Hamming distance.
\end{proposition}

Results similar to the following lemma are known.
\begin{lemma} [A basic KL upper-bound]
\label{lem:klbound} 
For any $0<p, q<1$, we let $\KLDiv{p}{q}$ 
denote $\KLDiv{\text{Ber}(p)}{\text{Ber}(q)}$. 
Now let $0<\epsilon<1/2$ and let $z\in \{ -1, 1\}$. We have 

$$\KLDiv{1/2 + (z/2)\cdot \epsilon\, }{\, 1/2 - (z/2)\cdot \epsilon} 
\leq c_0\cdot \epsilon^2, \text{ for some } c_0 \text{ independent of } \epsilon.$$
\end{lemma}
\begin{proof}
Write
$\frac{p}{q}\doteq \frac{1/2 + (z/2)\epsilon}{1/2 - (z/2)\epsilon} = 1 + \frac{2z\epsilon}{1 - z\epsilon}$ , and use the fact that 
$$\KLDiv{p}{q} \leq \chi^2(p| q) = q\paren{1-\frac{p}{q}}^2 + (1-q)\paren{1 - \frac{1-p}{1-q}}^2 = 
q\paren{\frac{2z\epsilon}{1 - z\epsilon}}^2 
+ (1-q)\paren{\frac{-2z\epsilon}{1 + z\epsilon}}^2.$$
\end{proof}

\begin{proof}[Proof of Theorem \ref{theo:lowrho}] 
Let $\VV = \V-1$.
Pick $x_0,x_1, x_2, \dots, x_{\VV}$ a shatterable subset of $\X$ under $\Hyp$. 
These will form the support of marginals $P_X, Q_X$. Furthermore, let $\tilde \Hyp$ denote the \emph{projection} of $\Hyp$ onto $\braces{x_i}_{i =0}^{\VV}$ (i.e., the quotient space of equivalences $h \equiv h'$ on $\braces{x_i}$),
with the additional constraint that all $h\in \tilde \Hyp$ classify $x_0$ as $1$. 
We can now restrict attention to $\tilde \Hyp$ as the \emph{effective} hypothesis class. 

Let $\sigma \in \braces{-1, 1}^{\VV}$. 
We will construct a family of distribution pairs $(P_\sigma, Q_\sigma)$ indexed by $\sigma$ to which we then apply Proposition \ref{prop:tsy25} above. 
For any $P_\sigma, Q_\sigma$, we let $\eta_{P, \sigma}, \eta_{Q, \sigma}$ denote the corresponding regression functions (i.e., $\expec_{P_\sigma} [Y | x]$, and $\expec_{Q_\sigma} [Y | x]$). To proceed, fix 
$\epsilon = c_1\cdot\epsilon(n_P, n_Q) \leq 1/2$, for a constant $c_1 < 1$ to be determined, where $\epsilon(n_P, n_Q)$ is as defined in the theorem's statement.

\emph{- Distribution $Q_\sigma$.} We have that $Q_\sigma = Q_X \times Q_{Y|X}^\sigma$, where 
$Q_X(x_0) = 1- \epsilon^{\beta_Q}$, while $Q_X(x_i) = \frac{1}{\VV}\epsilon^{\beta_Q}$, $i\geq 1$. 
Now, the conditional $Q_{Y|X}^\sigma$ is fully determined by $\eta_{Q, \sigma_i}(x_0) = 1$, 
and $\eta_{Q, \sigma}(x_i) = 1/2 + (\sigma_i/2)\cdot \epsilon^{1-\beta_Q}$, $i \geq 1$. 

\emph{- Distribution $P_\sigma$}. We have that $P_\sigma = P_X \times P_{Y|X}^\sigma$, $P_X(x_0) = 1- \epsilon^{\rho \beta_P}$, while $P_X(x_i) = \frac{1}{\VV}\epsilon^{\rho \beta_P}$, $i \geq 1$. Now, the conditional $P_{Y|X}^\sigma$ is fully determined by $\eta_{P, \sigma}(x_0) = 1$, and $\eta_{P, \sigma}(x_i) = 1/2 + (\sigma_i/2) \cdot \epsilon^{\rho(1-\beta_P)}$, $i \geq 1$. 

\emph{- Verifying that $(P_\sigma, Q_\sigma) \in \dist_\nc(\rho,\beta_P, \beta_Q, 1)$}. For any $\sigma \in \braces{-1, 1}^{\VV}$, let $h_\sigma \in \tilde \Hyp$ denote the corresponding
Bayes classifier (remark that the Bayes is the same for both $P_\sigma$ and $Q_\sigma$). Now, pick any other $h_{\sigma'} \in \tilde \Hyp$, and let $\text{dist}(\sigma, \sigma')$ denote the Hamming distance between $\sigma, \sigma'$ (as in Proposition \ref{lem:VGBound}). We then have that 
 \begin{align*}& \E_{Q_\sigma}(h_{\sigma'}) = \text{dist}(\sigma, \sigma')\cdot\frac{1}{\VV}\epsilon^{\beta_Q}\cdot \epsilon^{1-\beta_Q} = \frac{\text{dist}(\sigma, \sigma')}{\VV}\cdot \epsilon, \\
 &\text{ while } 
 Q_X(h_{\sigma'} \neq h_\sigma) = \frac{\text{dist}(\sigma, \sigma')}{\VV}\cdot \epsilon^{\beta_Q}, \\
 \text{and similarly, } 
 &\E_{P_\sigma}(h_{\sigma'}) = \frac{\text{dist}(\sigma, \sigma')}{\VV}\cdot \epsilon^\rho, \text{ while } 
P_X(h_{\sigma'} \neq h_\sigma) = \frac{\text{dist}(\sigma, \sigma')}{\VV}\cdot \epsilon^{\rho\beta_P}.
\end{align*}
{The condition is also easily verified for classifiers not labeling $x_0$ as $1$}.  
Since $({\text{dist}(\sigma, \sigma')}/{\VV}) \leq 1$, it follows that \eqref{eq:bern-class} holds with exponents $\beta_P$ and $\beta_Q$ for any $P_\sigma$ and $Q_\sigma$ respectively (with $C_{P_\sigma} = 1$, $C_{Q_\sigma} = 1$), and that any $P_\sigma$ admits a transfer-exponent $\rho$ w.r.t. $Q_\sigma$, with $C_\rho = 1$. 

\emph{- Reduction to a packing}. Now apply Proposition \ref{lem:VGBound} to identify a subset $\Sigma$ of $\braces{-1, 1}^{\VV}$, where 
$\abs{\Sigma} = M \geq 2^{\VV/8}$, and $\forall \sigma, \sigma' \in \Sigma$, we have $\text{dist}(\sigma, \sigma') \geq \VV/8$.  It should be clear then that 
for any $\sigma, \sigma' \in \Sigma$, 
$$\E_{Q_\sigma}(h_{\sigma'}) \geq \frac{\VV}{8}\cdot \frac{1}{\VV}\epsilon^{\beta_Q}\cdot \epsilon^{1-\beta_Q} = \epsilon/8.$$
Furthermore, by construction, any classifier $\hat h: \braces{x_i} \mapsto \braces{0, 1}$ can be reduced to a decision on $\sigma$, and we henceforth view $\text{dist}(\sigma, \sigma')$ as the semi-metric referenced in Proposition \ref{prop:tsy25}, with effective indexing set $\Sigma$. 

\emph{- KL bounds in terms of $n_P$ and $n_Q$}. 
Define $\Pi_\sigma = P_\sigma^{n_P}\times Q_\sigma^{n_Q}$. We can now verify that all $\Pi_\sigma, \Pi_{\sigma'}$ are close in KL-divergence. First notice that, for any $\sigma, \sigma' \in \Sigma$ (in fact in $\braces{-1,1}^{\VV}$)
\begin{align} 
\KLDiv{\Pi_\sigma}{\Pi_{\sigma'}} &= 
n_P \cdot \KLDiv{P_\sigma}{P_{\sigma'}} + 
n_Q \cdot \KLDiv{Q_\sigma}{Q_{\sigma'}} \nonumber \\
&= n_P \cdot \Expectation_{P_X} \KLDiv{P^\sigma_{Y|X}}{P^{\sigma'}_{Y|X}} + 
n_Q \cdot \Expectation_{Q_X} \KLDiv{Q^\sigma_{Y|X}}{Q^{\sigma'}_{Y|X}} \nonumber\\
&= n_P \cdot \sum_{i=1}^{\VV} \frac{\epsilon^{\rho \beta_P}}{\VV}\KLDiv{P^\sigma_{Y|x_i}}{P^{\sigma'}_{Y|x_i}} 
+ n_Q \cdot \sum_{i=1}^{\VV} \frac{\epsilon^{ \beta_Q}}{\VV}\KLDiv{Q^\sigma_{Y|x_i}}{Q^{\sigma'}_{Y|x_i}} \nonumber\\
&\leq c_0\paren{n_P\cdot \epsilon^{\rho(2-\beta_P)} + 
n_Q\cdot \epsilon^{(2-\beta_Q)}} \label{eq:firstkl}\\
&\leq c_0\VV(c_1^{\rho(2-\beta_p)} + c_1^{2-\beta_Q})
\leq 2c_0c_1 \VV.
\label{eq:finalkl}
\end{align}
where, for inequality \eqref{eq:firstkl}, we used Lemma \ref{lem:klbound} to upper-bound the divergence terms. It follows that, for $c_1$ sufficiently small so that 
$2c_0c_1 \leq 1/16$, we get that \eqref{eq:finalkl} is upper bounded by $(1/8) \log M$. Now apply Proposition \ref{prop:tsy25} and conclude. 
\end{proof}

We need the following lemma for the next result. 
\begin{lemma}
\label{lem:power}
Let $ \epsilon_1, \epsilon_2, \alpha, \alpha_1, \alpha_2\geq 0$, and $\alpha_1 + \alpha_2 \leq 1$. We then have that 
\begin{align*} 
&\text{ For }\alpha\geq 1, \quad \alpha_1\epsilon_1^\alpha + \alpha_2 \epsilon_2^\alpha 
\geq \paren{\alpha_1 \epsilon_1 + \alpha_2 \epsilon_2}^\alpha, \text{ and } \\
&\text{ for }\alpha \leq 1, \quad \alpha_1\epsilon_1^\alpha + \alpha_2 \epsilon_2^\alpha 
\leq \paren{\alpha_1 \epsilon_1 + \alpha_2 \epsilon_2}^\alpha.
\end{align*}

\end{lemma}
\begin{proof} 
W.l.o.g., let $\alpha_1 + \alpha_2 >0$, and normalize the l.h.s. of each of the above inequalities by $(\alpha_1 + \alpha_2)^{-1} \geq 1$. The results follows by Jensen's inequality and the convexity of 
$z \mapsto z^\alpha$ for $\alpha\geq 1$, and concavity of $z \mapsto z^\alpha$ for $\alpha\leq 1$. 
\end{proof}

We can now show Theorem \ref{theo:lowgamma}. 

\begin{proof}[Proof of Theorem \ref{theo:lowgamma}]

We proceed similarly (as far as high-level arguments) as for the proof of Theorem \ref{theo:lowrho}, but with a different construction where distributions now all satisfy
$\gamma = \rho\cdot \beta_P$, and are broken into two subfamilies (corresponding to the rates $\epsilon_1$ and $\epsilon_2$), and the final result holds by considering the intersection of these subfamilies. For simplicity, in what follows, assume $\VV$ is even, otherwise, the arguments hold by just replacing $\VV$ by $\VV -1$. First, 
define $x_0, x_1, x_2, \dots, x_{\VV}$, $\tilde \Hyp$ as in that proof. 

Let $\sigma \in \braces{-1, 1}^{\VV}$. Next we construct distribution pairs $P_\sigma, Q_\sigma$ indexed by $\sigma$, with corresponding regression functions $\eta_{P, \sigma}, \eta_{Q, \sigma}$. Fix 
$\epsilon_1 = c_1\cdot \epsilon_1(n_P, n_Q) \leq 1/2$, and $\epsilon_2 = c_2 \cdot \epsilon_2(n_P, n_Q) \leq 1/2$, for some $c_1, c_2 <1$ to be determined.



The construction is now 
broken up over  
$I_1 \doteq \braces{1, \ldots, \frac{\VV}{2}}$, and $I_2 \doteq \braces{\frac{\VV}{2} +1, \ldots, \VV}$. Fix a constant $\frac{1}{2}\leq \tau < 1$; this ensures that $\epsilon_2/\tau \leq 1$. We will later impose further conditions on $\tau$.  

\emph{- Distribution $Q_\sigma$.} We let $Q_\sigma = Q_X \times Q_{Y|X}^\sigma$, where 
$Q_X(x_0) = 1- \frac{1}{2}\paren{\epsilon_1^{\beta_Q} + (\epsilon_2/\tau)}$, while $Q_X(x_i) = \frac{1}{\VV}\epsilon_1^{\beta_Q}$ for $i \in I_1$, and $Q_X(x_i) = \frac{1}{\VV}(\epsilon_2/\tau)$ for $i \in I_2$. 
Now, the conditional $Q_{Y|X}^\sigma$ is fully determined by $\eta_{Q, \sigma}(x_0) = 1$, and $\eta_{Q, \sigma}(x_i) = 1/2 + (\sigma_i/2)\cdot \epsilon_1^{1-\beta_Q}$ for $i \in I_1$, and $\eta_{Q, \sigma}(x_i) = 1/2 + (\sigma_i/2)\cdot \tau$ for $i \in I_2$. 

\emph{- Distribution $P_\sigma$}. We let $P_\sigma = P_X \times P_{Y|X}^\sigma$, where $P_X(x_0) = 1- \frac{1}{2}\paren{\epsilon_1^{\gamma \beta_Q} + \epsilon_2^\gamma}$, while $P_X(x_i) = \frac{1}{\VV}\epsilon_1^{\gamma \beta_Q}$ for $i\in I_1$, and $P_X(x_i) = \frac{1}{\VV}\epsilon_2^{\gamma}$ for $i\in I_2$. Now, the conditional $P_{Y|X}^\sigma$ is fully determined by $\eta_{P, \sigma}(x_0) = 1$, and $\eta_{P, \sigma}(x_i) = 1/2 + (\sigma_i/2)\cdot \epsilon_1^{(1-\beta_P)\rho\beta_Q}$ for $i \in I_1$, and $\eta_{P, \sigma}(x_i) = 1/2 + (\sigma_i/2)\cdot \epsilon_2^{(1-\beta_P)\rho}$ for $i \in I_2$. 

\emph{- Verifying that $(P_\sigma, Q_\sigma) \in \dist_\nc(\rho,\beta_P, \beta_Q, 2)$}. For any $\sigma \in \braces{-1, 1}^{\VV}$, define $h_\sigma \in \tilde \Hyp$ as in the proof of Theorem \ref{theo:lowrho}.
Now, pick any other $h_{\sigma'} \in \tilde \Hyp$, and let $\text{dist}_I(\sigma, \sigma')$ denote the Hamming distance between $\sigma, \sigma'$, restricted to indices in $I$ (that is the Hamming distance between subvectors $\sigma_I$ and $\sigma_I'$). 
We then have that 
 \begin{align*} \E_{Q_\sigma}(h_{\sigma'}) &= \text{dist}_{I_1}(\sigma, \sigma')\cdot\frac{1}{\VV}\epsilon^{\beta_Q}\cdot \epsilon_1^{1-\beta_Q} + 
 \text{dist}_{I_2}(\sigma, \sigma')\cdot\frac{1}{\VV}(\epsilon_2/\tau)\tau \\
 &= \frac{\text{dist}_{I_1}(\sigma, \sigma')}{\VV}\epsilon_1 + 
 \frac{\text{dist}_{I_2}(\sigma, \sigma')}{\VV}\epsilon_2,\\
 \text{ while } 
 Q_X(h_{\sigma'} \neq h_\sigma) & = \frac{\text{dist}_{I_1}(\sigma, \sigma')}{\VV} \epsilon_1^{\beta_Q} + 
 \frac{\text{dist}_{I_2}(\sigma, \sigma')}{\VV} (\epsilon_2/\tau).\\
 \text{Similarly, } 
 \E_{P_\sigma}(h_{\sigma'}) & = \text{dist}_{I_1}(\sigma, \sigma')\cdot\frac{1}{\VV}\epsilon_1^{\gamma \beta_Q}\cdot\epsilon_1^{(1-\beta_P)\rho\beta_Q} + 
 \text{dist}_{I_2}(\sigma, \sigma')\cdot\frac{1}{\VV}\epsilon_2^\gamma \cdot 
 \epsilon_2^{(1-\beta_P) \rho},  \\
 & = \frac{\text{dist}_{I_1}(\sigma, \sigma')}{\VV}\epsilon_1^{\rho \beta_Q} + 
 \frac{\text{dist}_{I_2}(\sigma, \sigma')}{\VV}\epsilon_2^{\rho}, \\
 \text{ while } 
P_X(h_{\sigma'} \neq h_\sigma) &= \frac{\text{dist}_{I_1}(\sigma, \sigma')}{\VV} \epsilon_1^{\gamma\beta_Q} + 
 \frac{\text{dist}_{I_2}(\sigma, \sigma')}{\VV} \epsilon_2^\gamma.
\end{align*}
The condition is also easily verified for classifiers not labeling $x_0$ as $1$.
We apply Lemma \ref{lem:power} repeatedly in what follows. First, by the above, we have that 
\begin{align*}Q_X(h_{\sigma'} \neq h_\sigma) \leq 
\frac{\text{dist}_{I_1}(\sigma, \sigma')}{\VV}\epsilon_1^{\beta_Q} + 
 2\frac{\text{dist}_{I_2}(\sigma, \sigma')}{\VV}\epsilon_2^{\beta_Q}
\leq 2 \E_{Q_\sigma}^{\beta_Q}(h_{\sigma'}).
\end{align*}
On the other hand, 
\begin{align*}
  P_X(h_{\sigma'} \neq h_\sigma) = 
  \frac{\text{dist}_{I_1}(\sigma, \sigma')}{\VV}\paren{\epsilon_1^{\rho \beta_Q}}^{\beta_P} + 
 \frac{\text{dist}_{I_2}(\sigma, \sigma')}{\VV}\paren{\epsilon_2^{\rho}}^{\beta_P}
 \leq \E_{P_\sigma}^{\beta_P}(h_{\sigma'}),
\end{align*}
 Finally we have that 
$$ 
\E_{P_\sigma}(h_{\sigma'}) \geq 
\frac{\text{dist}_{I_1}(\sigma, \sigma')}{\VV}\epsilon_1^{\rho} + 
 \frac{\text{dist}_{I_2}(\sigma, \sigma')}{\VV}\epsilon_2^{\rho}
\geq \E_{Q_\sigma}^{\rho}(h_{\sigma'}).
$$

\emph{- Verifying that $\gamma$ is a marginal-transfer-exponent $P_X$ to  $Q_X$}. Using the above derivations, the condition that $\gamma\geq 1$, and further imposing the condition that $\tau\geq (1/2)^{1/\gamma}$, we have 
\begin{align*}
P_X(h_{\sigma'} \neq h_\sigma) \geq \frac{\text{dist}_{I_1}(\sigma, \sigma')}{\VV} \paren{\epsilon_1^{\beta_Q}}^\gamma + 
 \frac{1}{2}\frac{\text{dist}_{I_2}(\sigma, \sigma')}{\VV} (\epsilon_2/\tau)^\gamma
 \geq \frac{1}{2}Q_X^\gamma(h_{\sigma'} \neq h_\sigma).
\end{align*}
where we again used Lemma \ref{lem:power}. 

\emph{- Reduction to sub-Packings.} Now, in a slight deviation from the proof of Theorem \ref{theo:lowrho}, we define two separate packings (in Hamming distance), indexed by some $\varsigma$ as follows. Fix any $\varsigma \in \braces{-1, 1}^{\VV/2}$, and applying Proposition $2$, let 
$\Sigma_1(\varsigma) \subset \braces{\sigma \in \braces{-1, 1}^\VV: \sigma_{I_2} = \varsigma}$, 
and $\Sigma_2(\varsigma) \subset \braces{\sigma \in \braces{-1, 1}^\VV: \sigma_{I_1} = \varsigma}$ denote $m$-packings of $\braces{-1, 1}^{\VV/2}$, $m \geq \VV/16$, of size $M+1$, 
$M \geq 2^{\VV/16}$. 

Clearly, for any $\sigma, \sigma' \in \Sigma_1(\varsigma)$ we have 
$\E_{Q_\sigma}(h_{\sigma'}) \geq \epsilon_1/16$, while for any $\sigma, \sigma' \in \Sigma_2(\varsigma)$ we have $\E_{Q_\sigma}(h_{\sigma'}) \geq \epsilon_2/16$.

\emph{- KL Bounds in terms of $n_P$ and $n_Q$}. 
Again, define $\Pi_\sigma = P_\sigma^{n_P}\times Q_\sigma^{n_Q}$. 
First, for any $\varsigma$ fixed, let $\sigma, \sigma' \in \Sigma_1(\varsigma)$. As in the proof of Theorem \ref{theo:lowrho}, we apply Lemma \ref{lem:klbound} to get that  
\begin{align*} 
\KLDiv{\Pi_\sigma}{\Pi_{\sigma'}} &= 
 n_P \cdot \Expectation_{P_X} \KLDiv{P^\sigma_{Y|X}}{P^{\sigma'}_{Y|X}} + 
n_Q \cdot \Expectation_{Q_X} \KLDiv{Q^\sigma_{Y|X}}{Q^{\sigma'}_{Y|X}} \nonumber\\
&= n_P \cdot \sum_{i\in I_1} \frac{\epsilon_1^{\gamma \beta_Q}}{\VV}\KLDiv{P^\sigma_{Y|x_i}}{P^{\sigma'}_{Y|x_i}} 
+ n_Q \cdot \sum_{i\in I_1} \frac{\epsilon_1^{ \beta_Q}}{\VV}\KLDiv{Q^\sigma_{Y|x_i}}{Q^{\sigma'}_{Y|x_i}}  \nonumber\\
&\leq n_P\cdot c_0\frac{1}{2}\epsilon_1^{(2-\beta_P)\rho\beta_Q}  + 
n_Q\cdot c_0\frac{1}{2}\epsilon_1^{(2-\beta_Q)}  \nonumber\\
&\leq c_0\frac{\VV}{2}(c_1^{(2-\beta_P)\rho\beta_Q} + c_1^{2-\beta_Q})
\leq c_0c_1 \VV.
\end{align*}
Similarly, for any $\varsigma$ fixed, let $\sigma, \sigma' \in \Sigma_2(\varsigma)$; expanding over $I_2$, we have: \begin{align*}
\KLDiv{\Pi_\sigma}{\Pi_{\sigma'}} &\leq  n_P\cdot c_0\frac{1}{2}\epsilon_2^{(2-\beta_P)\rho}  + 
n_Q\cdot c_0\frac{1}{2}\epsilon_2\cdot\tau
\leq c_0 c_1 \VV. 
\end{align*}
It follows that, for $c_1$ sufficiently small so that 
$c_0c_1 \leq 1/16$, we can apply Proposition \ref{prop:tsy25} twice, to get that \emph{for all} $\varsigma$, there exist $\sigma_{I_1}$ and $\sigma_{I_2}$, such that for some constant $c$, we have 
$$\expec_{\Pi_\sigma}\paren{\E_{Q_\sigma}(\hat h)} \geq c\cdot \epsilon_1, \text{ where } \sigma = [\sigma_{I_1}, \varsigma], \text{ and } 
\expec_{\Pi_\sigma}\paren{\E_{Q_\sigma}(\hat h)} \geq c\cdot \epsilon_2, \text{ where } \sigma = [\varsigma, \sigma_{I_2}].$$
It follows that $c\cdot\max\braces{\epsilon_1, \epsilon_2}$ is a lower-bound for either 
$\sigma = [\sigma_{I_1}, \varsigma]$ or $\sigma = [\varsigma, \sigma_{I_2}]$. 
\end{proof}

\section{Upper Bounds Proofs}
\label{app:lepski}

\begin{proof}[Proof of Proposition~\ref{prop:lepski-agnostic}]
To reduce redundancy, we refer to arguments 
presented in the proof of Theorem~\ref{thm:lepski}, 
rather than repeating them here.
As in the proof of Theorem~\ref{thm:lepski}, 
we let $C$ serve as a generic constant (possibly depending on $\rho',C_{\rho'},\beta_{P},c_{\beta_{P}},\beta_{Q},c_{\beta_{Q}}$) which may be different in different appearances.
Define a set
\begin{equation*}
\Gyp = \left\{ h \in \H : \hat{R}_{S_{Q}}(h) - \hat{R}_{S_{Q}}(\hat{h}_{S_{Q}}) \leq c \sqrt{ \hat{P}_{S_{Q}}( h \neq \hat{h}_{S_{Q}} ) A_{n_{Q}} } + c A_{n_{Q}} \right\}.
\end{equation*}
We can rephrase the definition of $\hat{h}$ as saying 
$\hat{h} = \hat{h}_{S_{P}}$ when $\hat{h}_{S_{P}} \in \Gyp$, 
and otherwise $\hat{h} = \hat{h}_{S_{Q}}$.

We suppose the event from Lemma~\ref{lem:vc-bernstein} holds 
for both $S_Q$ and $S_P$; by the union bound, this happens with probability at least $1-\delta$.
In particular, as in \eqref{eqn:Gyp-eqn} from the proof of Theorem~\ref{thm:lepski}, we have 
\begin{equation*}
\E_P(\hat{h}_{S_{P}}) \leq C A_{n_{P}}^{\frac{1}{2-\beta_{P}}}.
\end{equation*}
Together with the definition of $\rho^{\prime}$, this implies
\begin{equation*}
\E_Q(\hat{h}_{S_{P}},h^*_P) \leq C A_{n_{P}}^{\frac{1}{(2-\beta_{P})\rho^{\prime}}}, 
\end{equation*}
which means
\begin{equation}
\label{eqn:EQ-additive-bound}
\E_Q(\hat{h}_{S_{P}}) \leq \E_Q(h^*_P) + \E_Q(\hat{h}_{S_{P}},h^*_P) 
\leq \E_Q(h^*_P) + C A_{n_{P}}^{\frac{1}{(2-\beta_{P})\rho^{\prime}}}.
\end{equation}
Now, if $R_{Q}(\hat{h}_{S_{P}}) \leq R_{Q}(\hat{h}_{S_{Q}})$, 
then (due to the event from Lemma~\ref{lem:vc-bernstein}) we have 
$\hat{h}_{S_{P}} \in \Gyp$, 
so that $\hat{h} = \hat{h}_{S_{P}}$, 
and thus the rightmost expression in \eqref{eqn:EQ-additive-bound}
bounds $\E_Q(\hat{h})$.
On the other hand, if $R_{Q}(\hat{h}_{S_{P}}) > R_{Q}(\hat{h}_{S_{Q}})$,
then regardless of whether $\hat{h} = \hat{h}_{S_{P}}$ or $\hat{h} = \hat{h}_{S_{Q}}$, 
we have $\E_Q(\hat{h}) \leq  \E_Q(\hat{h}_{S_{P}})$, 
so that again the rightmost expression in \eqref{eqn:EQ-additive-bound} bounds $\E_Q(\hat{h})$.
Thus, in either case, 
\begin{equation*}
\E_Q(\hat{h}) \leq \E_Q(h^*_P) + C A_{n_{P}}^{\frac{1}{(2-\beta_{P})\rho^{\prime}}}.
\end{equation*}

Furthermore, as in the proof of Theorem~\ref{thm:lepski}, 
every $h \in \Gyp$ satisfies 
$\E_Q(h) \leq C A_{n_{Q}}^{\frac{1}{2-\beta_{Q}}}$.
Since the algorithm only picks $\hat{h} = \hat{h}_{S_{P}}$ 
if $\hat{h}_{S_{P}} \in \Gyp$, and otherwise picks $\hat{h} = \hat{h}_{S_{Q}}$,
which is clearly in $\Gyp$, 
we may note that we always have $\hat{h} \in \Gyp$.
We therefore conclude that 
\begin{equation*}
\E_Q(\hat{h}) \leq C A_{n_{Q}}^{\frac{1}{2-\beta_{Q}}},
\end{equation*}
which completes the proof.
\end{proof}

\section{Proofs for Adaptive Sampling Costs}
\label{app:adaptive}

\begin{proof}[Proof of Theorem~\ref{thm:adaptive-cost}]
First note that since $\sum_{n} \frac{1}{2n^2} < 1$, 
by the union bound and Lemma~\ref{lem:vc-bernstein}, 
with probability at least $1-\delta$, for every $h,h^{\prime} \in \Hyp$, 
every set $S_{P}$ in the algorithm has 
\begin{equation*}
R_{P}(h) - R_{P}(h') \leq \hat{R}_{S_{P}}(h) - \hat{R}_{S_{P}}(h') + c \sqrt{ \min\{ P(h \neq h'), \hat{P}_{S_{P}}( h \neq h' ) \} A_{|S_{P}|}^{\prime}} + c A_{|S_{P}|}^{\prime}
\end{equation*}
and 
\begin{equation*}
\hat{P}_{S_{P}}(h \neq h') \leq 2 P( h \neq h' ) + c A_{|S_{P}|}^{\prime}
\end{equation*}
every set $S_{Q}$ in the algorithm has 
\begin{equation*}
R_{Q}(h) - R_{Q}(h') \leq \hat{R}_{S_{Q}}(h) - \hat{R}_{S_{Q}}(h') + c \sqrt{ \min\{ Q(h \neq h'), \hat{P}_{S_{Q}}( h \neq h' ) \} A_{|S_{Q}|}^{\prime}} + c A_{|S_{Q}|}^{\prime}
\end{equation*}
and
\begin{equation*}
\hat{P}_{S_{Q}}(h \neq h') \leq 2 Q( h \neq h' ) + c A_{|S_{Q}|}^{\prime},
\end{equation*}
and we also have for the set $U_{Q}$ that 
\begin{equation*}
\frac{1}{2} Q(h \neq h') - c A_{|U_{Q}|} \leq \hat{P}_{U_{Q}}(h \neq h') \leq 2 Q(h \neq h') + c A_{|U_{Q}|},
\end{equation*}
which by our choice of the size of $U_{Q}$ implies 
\begin{equation*}
\frac{1}{2} Q(h \neq h') - \frac{\epsilon}{8} \leq \hat{P}_{U_{Q}}(h \neq h') \leq 2 Q(h \neq h') + \frac{\epsilon}{8}.
\end{equation*}
For the remainder of this proof, we suppose these inequalities hold.

In particular, these imply  
\begin{equation*}
R_{Q}(\hat{h}_{S_{Q}}) - R_{Q}(h^{*}) \leq c \sqrt{ \hat{P}_{S_{Q}}( \hat{h}_{S_{Q}} \neq h^{*} ) A_{|S_{Q}|}^{\prime} } + c A_{|S_{Q}|}^{\prime}.
\end{equation*}
Furthermore, 
\begin{equation*}
\hat{R}_{S_{Q}}(h^{*}) - \hat{R}_{S_{Q}}(\hat{h}_{S_{Q}}) \leq c \sqrt{ \hat{P}_{S_{Q}}( h^{*} \neq \hat{h}_{S_{Q}} ) A_{|S_{Q}|}^{\prime} } + c A_{|S_{Q}|}^{\prime}, 
\end{equation*}
so that $h = h^{*}$ is included in the supremum in the definition of $\hat{\delta}(S_{Q},S_{Q})$.
Together these imply 
\begin{equation*}
\E_{Q}(\hat{h}_{S_{Q}}) \leq R_{Q}(\hat{h}_{S_{Q}}) - R_{Q}(h^{*}) \leq c \sqrt{ \hat{\delta}(S_{Q},S_{Q}) A_{|S_{Q}|} } + c A_{|S_{Q}|}.
\end{equation*}
Thus, if the algorithm returns $\hat{h}_{S_{Q}}$ in Step 6, then $\E_{Q}(\hat{h}_{S_{Q}}) \leq \epsilon$.

Also by the above inequalities, we have 
\begin{equation*}
\hat{R}_{S_{P}}(h^{*}) - \hat{R}_{S_{P}}(\hat{h}_{S_{P}}) \leq c \sqrt{ \hat{P}_{S_{P}}( h^{*} \neq \hat{h}_{S_{Q}} ) A_{|S_{P}|}^{\prime} } + c A_{|S_{P}|}^{\prime}, 
\end{equation*}
so that $h^{*}$ is included in the supremum in the definition of $\hat{\delta}(S_{P},U_{Q})$.
Thus, 
\begin{equation*}
\E_{Q}(\hat{h}_{S_{P}}) \leq Q( \hat{h}_{S_{P}} \neq h^{*} ) \leq 2 \hat{P}_{U_{Q}}( \hat{h}_{S_{P}} \neq h^{*} ) + \frac{\epsilon}{2} \leq 2 \hat{\delta}(S_{P},U_{Q}) + \frac{\epsilon}{2},
\end{equation*}
and hence if the algorithm returns $\hat{h}_{S_{P}}$ in Step 7 we have $\E_{Q}(\hat{h}_{S_{P}}) \leq \epsilon$ as well.
Furthermore, the algorithm will definitely return at some point, since the bound in Step 6 approaches $0$ as the sample size grows.
Altogether, this establishes that, on the above event, the $\hat{h}$ returned by the algorithm satisfies $\E_{Q}(\hat{h}) \leq \epsilon$, as claimed.

It remains to show that the cost satisfies the stated bound.
For this, first note that since the costs incurred by the algorithm grow as a function that is upper and lower bounded by a geometric series, 
it suffices to argue that, for an appropriate choice of the constant $c^{\prime}$, the algorithm would halt if ever it reached 
a set $S_{P}$ of size at least $n_{P}^{*}$ or a set $S_{Q}$ of size at least $n_{Q}^{*}$ (which ever were to happen first); the result would then 
follow by choosing the actual constant $c^{\prime}$ in the theorem slightly larger than this, to account for the algorithm slighly ``overshooting'' this target (by at most a numerical constant factor).

First suppose it reaches $S_{Q}$ of size at least $n_{Q}^{*}$.
Now, as in the proof of Theorem~\ref{thm:lepski}, on the above event, 
every $h \in \Hyp$ included in the supremum in the definition of $\hat{\delta}(S_{Q},S_{Q})$ has
\begin{equation*}
\E_{Q}(h) \leq C \left( A_{|S_{Q}|}^{\prime} \right)^{\frac{1}{2-\beta_{Q}}},
\end{equation*}
which further implies 
\begin{equation*}
Q( h \neq h^{*} ) \leq C \left( A_{|S_{Q}|}^{\prime} \right)^{\frac{\beta_{Q}}{2-\beta_{Q}}},
\end{equation*}
so that (by the triangle inequality and the above inequalities) 
\begin{equation*}
\hat{P}_{S_{Q}}( h \neq \hat{h}_{S_{Q}} ) \leq C \left( A_{|S_{Q}|}^{\prime} \right)^{\frac{\beta_{Q}}{2-\beta_{Q}}}.
\end{equation*}
Thus, in Step 6, 
\begin{equation*}
c \sqrt{ \hat{\delta}(S_{Q},S_{Q}) A_{|S_{Q}|} } + c A_{|S_{Q}|} 
\leq C \left( A_{|S_{Q}|}^{\prime} \right)^{\frac{1}{2-\beta_{Q}}},
\end{equation*}
which, by our choice of $n_{Q}^{*}$ is at most $\epsilon$.  Hence, in this case, the algorithm will return in Step 6 (or else would have returned on some previous round).

On the other hand, suppose $S_{P}$ reaches a size at least $n_{P}^{*}$.
In this case, again by the same argument used in the proof of Theorem~\ref{thm:lepski}, 
every $h \in \Hyp$ included in the supremum in the definition of $\hat{\delta}(S_{P},U_{Q})$ has 
\begin{equation*}
\E_{P}(h) \leq C \left( A_{|S_{P}|}^{\prime} \right)^{\frac{1}{2-\beta_{P}}},
\end{equation*}
which implies 
\begin{equation*}
P( h \neq h^{*} ) \leq C \left( A_{|S_{P}|}^{\prime} \right)^{\frac{\beta_{P}}{2-\beta_{P}}},
\end{equation*}
and hence 
\begin{equation*}
Q( h \neq h^{*} ) \leq C \left( A_{|S_{P}|}^{\prime} \right)^{\frac{\beta_{P}}{(2-\beta_{P}) \gamma}}.
\end{equation*}
By the above inequalities
and the triangle inequality (since $\hat{h}_{S_{P}}$ is clearly also included as an $h$ in that supremum), 
this implies 
\begin{equation*}
\hat{P}_{U_{Q}}( h \neq \hat{h}_{S_{P}} ) \leq C \left( A_{|S_{P}|}^{\prime} \right)^{\frac{\beta_{P}}{(2-\beta_{P}) \gamma}} + \frac{\epsilon}{8}.
\end{equation*}
Altogether we get that 
\begin{equation*}
\hat{\delta}(S_{P},U_{Q}) \leq C \left( A_{|S_{P}|}^{\prime} \right)^{\frac{\beta_{P}}{(2-\beta_{P}) \gamma}} + \frac{\epsilon}{8}.
\end{equation*}
By our choice of $n_{P}^{*}$ (for an appropriate choice of constant factors), the right hand side is at most $\epsilon/4$.
Therefore, in this case the algorithm will return in Step 7 (if it had not already returned in some previous round).
This completes the proof.
\end{proof}

\section{Proofs for Reweighting Results}
\label{app:reweighting}

The following lemma is known (see \cite{van-der-Vaart:98,hanneke:12}), 
following from the general form of Bernstein's inequality and standard VC arguments, 
in combination with the well-known fact that, since the VC dimension of 
$\{ (x,y) \mapsto \ind[ h(x) \neq y ] : h \in \Hyp \}$ is $\V$, 
and pseudo-dimension of $\Pclass$ is $\Pdim$, 
it follows that the pseudo-dimension of $\{ (x,y) \mapsto \ind[ h(x) \neq y ] f(x) : h \in \Hyp, f \in \Pclass \}$
is at most $\propto \V + \Pdim$.

\begin{lemma}
\label{lem:reweighted-bernstein}
With probability at least $1-\frac{\delta}{3}$, 
$\forall f \in \Pclass$, 
$\forall h,h' \in \Hyp$, 
\begin{equation*}
R_{P_{f}}\!(h) - R_{P_{f}}\!(h^{\prime}) \!\leq\! \hat{R}_{S_{P},f}(h) - \hat{R}_{S_{P},f}(h^{\prime}) + c \sqrt{ \min\{ \!P_{f^{2}}(h \!\neq\! h^{\prime}), \!\hat{P}_{S_{P},f^{2}}( h \!\neq\! h^{\prime} ) \!\} \!A_{n_{P}}^{\prime\prime} } + c \|f\|_{\infty} A_{n_{P}}^{\prime\prime}
\end{equation*}
and
$\frac{1}{2} P_{f^{2}}( h \neq h^{\prime} ) - c \|f\|_{\infty} A_{n_{P}}^{\prime\prime} \leq \hat{P}_{S_{P},f^{2}}( h \neq h^{\prime} ) \leq 2 P_{f^{2}}( h \neq h^{\prime} ) + c \|f\|_{\infty} A_{n_{P}}^{\prime\prime}$,
for a universal numerical constant $c \in (0,\infty)$.
\end{lemma}

\begin{proof}[Proof of Theorem~\ref{thm:reweighting-finite}]
Let us suppose the event from Lemma~\ref{lem:reweighted-bernstein} holds, 
as well as the event from Lemma~\ref{lem:vc-bernstein} for $S_{Q}$, 
and also the part \eqref{eqn:vc-empirical-distance-bound} from the event in Lemma~\ref{lem:vc-bernstein} holds for $U_{Q}$.
The union bound implies all of these hold simultaneously with probability at least $1-\delta$.
For simplicity, and without loss of generality, we will suppose the constants $c$ in these two lemmas are the same.
Regarding the sufficient size of $|U_{Q}|$, for this result it suffices to have $|U_{Q}| \geq n_{P}^{\frac{\beta_{f}}{(2-\beta_{f})\gamma_{f}}}$ for all $f \in \Pclass$; 
for instance, in the typical case where $\gamma_{f} \geq 1$ for all $f \in \Pclass$, it would suffice to simply have $|U_{Q}| \geq n_{P}$.

First note that, exactly as in the proof of Theorem~\ref{thm:lepski}, 
since the event in Lemma~\ref{lem:reweighted-bernstein} implies $h^{*}_{P_{\hat{f}}}$ satisfies the constraint in the optimization defining $\hat{h}$, 
and the RCS assumption implies $\E_{Q}(h^{*}_{P_{\hat{f}}}) = 0$, and hence by \nc\ that $Q(h^{*}_{P_{\hat{f}}} \neq h^{*}_{Q}) = 0$, 
we immediately get that 
\begin{equation*}
\E_{Q}(\hat{h}) \leq C A_{n_{Q}}^{\frac{1}{2-\beta_{Q}}}.
\end{equation*}
Thus, it only remains to establish the other term in the minimum as a bound.

Similarly to the proofs above, we let $C_{f}$ be a general $f$-dependent constant (with the same restrictions on dependences mentioned in the theorem statement), 
which may be different in each appearance below.
For each $f \in \Pclass$, denote by $\hat{h}_{f}$ the $h \in \Hyp$ 
that minimizes $\hat{R}_{S_{Q}}(h)$ among $h \in \Hyp$ 
subject to $\hat{\E}_{S_{P},f}(h) \leq c \sqrt{ \hat{P}_{S_{P},f^{2}}( h \neq \hat{h}_{S_{P},f} ) A_{n_{P}}^{\prime\prime} } + c \| f \|_{\infty} A_{n_{P}}^{\prime\prime}$.
Also note that $\hat{h}_{S_{P},f}$ certainly satisfies the constraint in the set defining $\hat{\delta}(S_{P},f,U_{Q})$, 
and that the event from Lemma~\ref{lem:reweighted-bernstein} implies $h^{*}_{P_{f}}$ also satisfies this same constraint.
Therefore, the event for $U_{Q}$ from Lemma~\ref{lem:vc-bernstein}, and the triangle inequality, imply  
\begin{equation*}
\E_{Q}(\hat{h}_{f}) 
\leq Q( \hat{h}_{f} \neq h^{*}_{P_{f}} ) 
\leq Q( \hat{h}_{f} \neq \hat{h}_{S_{P},f} ) + Q( h^{*}_{P_{f}} \neq \hat{h}_{S_{P},f} ) 
\leq 4 \hat{\delta}(S_{P},f,U_{Q}) + 4 c A_{|U_{Q}|}.
\end{equation*}
Thus, $\hat{f}$ is being chosen to minimize an upper bound on the excess $Q$-risk of the resulting classifier.

Next we relax this expression to match that in the theorem statement.
Again using \eqref{eqn:vc-empirical-distance-bound}, we get that 
\begin{multline*}
\hat{\delta}(S_{P},f,U_{Q}) 
\leq c A_{|U_{Q}|} + \\
2 \sup\!\left\{ Q( h \neq \hat{h}_{S_{P},f} ) : h \in \Hyp, \hat{\E}_{S_{P},f} \leq c \sqrt{\hat{P}_{S_{P},f^{2}}( h \neq \hat{h}_{S_{P},f} ) A_{n_{P}}^{\prime\prime}} + c \|f\|_{\infty} A_{n_{P}}^{\prime\prime} \right\}.
\end{multline*}
Again since $h^{*}_{P_{f}}$ and $\hat{h}_{S_{P},f}$ both satisfy the constraint in this set, the supremum on the right hand side is at most 
\begin{equation*}
2 \sup\!\left\{ Q( h \neq h^{*}_{P_{f}} ) : h \in \Hyp, \hat{\E}_{S_{P},f} \leq c \sqrt{\hat{P}_{S_{P},f^{2}}( h \neq \hat{h}_{S_{P},f} ) A_{n_{P}}^{\prime\prime}} + c \|f\|_{\infty} A_{n_{P}}^{\prime\prime} \right\}.
\end{equation*}
Then using the marginal transfer condition, this is at most 
\begin{equation*}
C_{f} \sup\!\left\{ P_{f^{2}}( h \neq h^{*}_{P_{f}} )^{\frac{1}{\gamma_{f}}} : h \in \Hyp, \hat{\E}_{S_{P},f} \leq c \sqrt{\hat{P}_{S_{P},f^{2}}( h \neq \hat{h}_{S_{P},f} ) A_{n_{P}}^{\prime\prime}} + c \|f\|_{\infty} A_{n_{P}}^{\prime\prime} \right\},
\end{equation*}
and the Bernstein Class condition further bounds this as 
\begin{equation*}
C_{f} \sup\!\left\{ \E_{P_{f}}^{\frac{\beta_{f}}{\gamma_{f}}} : h \in \Hyp, \hat{\E}_{S_{P},f} \leq c \sqrt{\hat{P}_{S_{P},f^{2}}( h \neq \hat{h}_{S_{P},f} ) A_{n_{P}}^{\prime\prime}} + c \|f\|_{\infty} A_{n_{P}}^{\prime\prime} \right\}.
\end{equation*}
Finally, by essentially the same argument as in the proof of Theorem~\ref{thm:lepski} above, 
every $h \in \Hyp$ with $\hat{\E}_{S_{P},f} \leq c \sqrt{\hat{P}_{S_{P},f^{2}}( h \neq \hat{h}_{S_{P},f} ) A_{n_{P}}^{\prime\prime}} + c \|f\|_{\infty} A_{n_{P}}^{\prime\prime}$ 
satisies 
\begin{equation*}
\E_{P_{f}}(h) \leq C_{f} (A_{n_{P}}^{\prime\prime})^{\frac{1}{2-\beta_{f}}},
\end{equation*}
so that the above supremum is at most $C_{f} (A_{n_{P}}^{\prime\prime})^{\frac{\beta_{f}}{(2-\beta_{f})\gamma_{f}}}$ for a (different) appropriate choice of $C_{f}$.
Altogether we have established that 
\begin{equation*}
\hat{\delta}(S_{P},f,U_{Q}) 
\leq c A_{|U_{Q}|} + C_{f} (A_{n_{P}}^{\prime\prime})^{\frac{\beta_{f}}{(2-\beta_{f})\gamma_{f}}}.
\end{equation*}
By our condition on $|U_{Q}|$ specified above, this implies 
\begin{equation*}
\hat{\delta}(S_{P},f,U_{Q}) \leq C_{f} (A_{n_{P}}^{\prime\prime})^{\frac{\beta_{f}}{(2-\beta_{f})\gamma_{f}}}.
\end{equation*}

We therefore have that 
\begin{align*}
\E_{Q}(\hat{h}) 
&= \E_{Q}(\hat{h}_{\hat{f}})
\leq 4 \hat{\delta}(S_{P},\hat{f},U_{Q}) + 4 c A_{|U_{Q}|} 
= \inf_{f \in \Pclass} 4 \hat{\delta}(S_{P},\hat{f},U_{Q}) + 4 c A_{|U_{Q}|} \\
&\leq \inf_{f \in \Pclass} C_{f} (A_{n_{P}}^{\prime\prime})^{\frac{\beta_{f}}{(2-\beta_{f})\gamma_{f}}},
\end{align*}
where we have again used the condition on $|U_{Q}|$.
This completes the proof.
\end{proof}

\end{document}